\documentclass{article}
\usepackage{graphicx}
\usepackage{subfigure} 
\usepackage{natbib}
\usepackage{algorithm}
\usepackage{algorithmic}
\usepackage[hypertex]{hyperref} 

\usepackage[cmex10]{amsmath}
\usepackage{amssymb} 
\usepackage{amsfonts}
\usepackage{amsthm}
\usepackage{bm}
\usepackage{color}
\usepackage{url}
\newcommand{\red}[1]{#1} 

\newtheorem{thm}{Theorem}
\newtheorem{lem}{Lemma}

\newtheorem{rmk}{Remark}

\newcommand{\mb}[1]{\mbox{\boldmath $#1$}}

\newcommand{\valpha}[0]{\mb{\alpha}}
\newcommand{\vbeta}[0]{\mb{\beta}}

\newcommand{\vxi}[0]{\mb{\xi}}
\newcommand{\vmu}[0]{\mb{\mu}}
\newcommand{\vnu}[0]{\mb{\nu}}

\newcommand{\vQ}[0]{\mb{Q}}
\newcommand{\vM}[0]{\mb{M}}

\newcommand{\vc}[0]{\mb{c}}
\newcommand{\vd}[0]{\mb{d}}
\newcommand{\vg}[0]{\mb{g}}

\newcommand{\vp}[0]{\mb{p}}
\newcommand{\vq}[0]{\mb{q}}
\newcommand{\vw}[0]{\mb{w}}
\newcommand{\vx}[0]{\mb{x}}
\newcommand{\vy}[0]{\mb{y}}
\newcommand{\vz}[0]{\mb{z}}

\newcommand{\tvQ}{\tilde{\vQ}}
\newcommand{\tvy}{\tilde{\vy}}
\newcommand{\tvg}{\tilde{\mb{g}}}

\newcommand{\tvbeta}{\tilde{\vbeta}}

\newcommand{\wtD}[0]{\widetilde{D}}

\newcommand{\wtvalpha}[0]{\widetilde{\valpha}}

\newcommand{\whvbeta}[0]{\widehat{\vbeta}}

\newcommand{\whg}[0]{\widehat{g}}
\newcommand{\whvg}[0]{\widehat{\vg}}
\newcommand{\whbeta}[0]{\widehat{\beta}}

\newcommand{\pd}[2]{{\frac{\partial #1}{\partial #2}}}

\newcommand{\lw}[1]{\smash{\lower2.ex\hbox{#1}}}

\newcommand{\RR}{\mathbb{R}}

\newcommand{\cB}{{\cal B}}

\newcommand{\cF}{{\cal F}}

\newcommand{\cI}{{\cal I}}

\newcommand{\cM}{{\cal M}}

\newcommand{\cO}{{\cal O}}

\newcommand{\cS}{{\cal S}}
\newcommand{\cT}{{\cal T}}

\newcommand{\cX}{{\cal X}}

\title{Suboptimal Solution Path Algorithm for \\ Support Vector Machine}
\author{Masayuki Karasuyama, Ichiro Takeuchi \\
Nagoya Institute of Technology}
\date{}
\begin{document} 
\maketitle
 
\begin{abstract} 
We consider a \emph{suboptimal solution path algorithm} for  the
Support Vector Machine. The solution path algorithm is an
effective tool for solving a sequence of a parametrized
optimization problems in machine learning. The path of the
solutions provided by this algorithm are very accurate and they
satisfy the optimality conditions more strictly than other SVM
optimization algorithms. In many machine learning application,
however, this strict optimality is often unnecessary, and it
adversely affects the computational efficiency. Our algorithm can
generate the path of suboptimal solutions within an arbitrary
user-specified tolerance level. It allows us to control the
trade-off between the accuracy of the solution and the
computational cost. Moreover, We also show that our suboptimal
solutions can be interpreted as the solution of a \emph{perturbed
optimization problem} from the original one. We provide some
theoretical analyses of our algorithm based on this novel
interpretation. The experimental results also demonstrate the
effectiveness of our algorithm.
\end{abstract}

\section{Introduction}
\label{sec:intro}

Recently, the \emph{solution path algorithm}~\citep{Efron04, Hastie04, Cauwenberghs01} has been widely recognized as one of the effective tools in machine learning.
It can efficiently compute a sequence of the solutions of a parametrized optimization problem.
This technique is originally developed as \emph{parametric programming} in the optimization community~\citep{UWaterloo:Best:1982}.

In a class of parametric quadratic programs (QPs), the solution path is represented as a piecewise-linear function of the problem parameters.
If we regard the regularization parameter of the Support Vector Machine (SVM) as problem parameter, the optimization problem for the SVM is categorized in this class. 
Therefore, the SVM solutions are represented as piecewise-linear functions of the regularization parameter.

The solutions of these parametric QPs are characterized by active constraint set in the current solution.
The linearity of the path comes from the fact that the Karush-Khun-Tucker (KKT) optimality conditions of these problems are represented as a linear system defined by the current active set, 
while 
the ``piecewise-ness'' is the consequence of the changes in the active set. 
The piecewise-linear solution path algorithm repeatedly updates the
linear system and active set.
The point of active set change is called \emph{breakpoint} in the
literature.
The path of solutions generated by this algorithm is very accurate and
they satisfy the optimality conditions more strictly than other algorithms.

Many machine learning problems, however, do not require strict optimality of the solution. 
In fact, one of the popular SVM optimization algorithm, called sequential minimal optimization (SMO)~\cite{Platt99}, is known to produce suboptimal (approximated) solution, where the tolerance to the optimality (degree of approximated) can be specified by users. 
In many experimental studies, it has been demonstrated that the generalization performances of these suboptimal solutions are not significantly different from those of strictly optimal ones. 

Therefore, the strict optimality of the solution path algorithm is often unnecessary.
Furthermore, it adversely affects the computational efficiency of the algorithm.
In fact, the solution path algorithm can be very slow when it encounters a large number of (seemingly redundant) breakpoints. 
Although some empirical studies suggest that the number of breakpoints grows linearly in the input size, in the worst case,  it can grow exponentially \citep{Gartner10}.
Another difficulty is in starting the solution path algorithm from
an approximated solution, for example obtained by SMO, because it does
not satisfy the strict optimality requirement.

In order to address these issues in the current solution path algorithm, we introduce a \emph{suboptimal solution path algorithm}. 
Our algorithm also generates piecewise-linear solution path, but the optimality tolerance (approximation level) can be arbitrary controlled by users.
It allows to control the trade-off between the accuracy of the solution and the computational cost. 

The presented suboptimal solution path algorithm has the following properties. 
\begin{itemize}
 \item First, the algorithm can reduce the number of breakpoints (which is the main computational bottleneck in solution path algorithm)
       by allowing multiple active set changes at one breakpoint. 
       Although this modification causes what is called \emph{degeneracy} problem, 
       we provide an efficient and accurate way to solve this issue. 
       We empirically show that reducing the number of breakpoints can work effectively to the computational efficiency.

 \item Second, the suboptimal solutions obtained by the algorithm can be interpreted as the solution of a \emph{perturbed optimization problem} from the original one. 
       This novel interpretation provides several insights into the properties of our suboptimal solutions. 
       We present some theoretical analyses of our suboptimal solutions using this interpretation.
\end{itemize}

We also empirically investigate several practical properties of our approach.
Although, our algorithm updates multiple active constraints at one
breakpoint, we observe that the entire changing patterns of the active
sets are very similar to those of the exact path.
Moreover, despite its computational efficiency, the generalization performance of our suboptimal path is comparable to conventional one.

To the best of our knowledge, there are no previous works for suboptimal
solution path algorithm with controllable optimality tolerance that can
be applicable to standard SVM formulation~\footnote{
\citet{Giesen10} proposed approximated path algorithm with some optimality
guarantee that can be applicable to L2-SVM without bias term. 
}.
Although many authors mimic the solution path by just repeating the
warm-start on finely grid points \citep[e.g.,][]{Friedman07}, this approach does not provide any guarantee about the intermediate solutions between grid points.
In this paper we focus our attention to the solution path algorithm for standard SVM, but the presented approach can be applied to other problems in the aforementioned QP class.


\section{Solution Path for Support Vector Machine}
\label{sec:svm_path}

In this section, we describe the solution path algorithm 
for regularization parameters of Support
Vector Machine (SVM).

\subsection{Support Vector Machine}

Suppose we have a set of training data 
$\{(\vx_i, y_i)\}_{i = 1}^n$,
where 
$\vx_i \in \cX \subseteq \RR^p$
is the input and 
$y_i \in \{-1, +1\}$ 
is the output class label.
SVM learns a linear discriminant function 
$f(\vx) = \vw^\top \Phi(\vx) + \alpha_0$
in a feature space $\cF$, where 
$\Phi: \cX \rightarrow \cF$
is a map from the input space $\cX$ to the feature space 
$\cF$, 
$\vw \in \cF$ is a coefficient vector and
$\alpha_0 \in \RR$
is a bias term.

In this paper, we consider the optimization problem of the following
form:
\begin{align}
   \min_{\vw, \alpha_0, \{\xi_i\}_{i=1}^n} ~~& 
 \textstyle{
 \frac{1}{2} \| \mb{w} \|_2^2 
   + \sum_{i=1}^n C_i \xi_i, }  \label{eq:svm.primal} \\
   {\rm s.t. }        ~~& y_i f(\bm{x}_i) \geq 1 - \xi_i, 
   \  \xi_i \geq 0, \ i = 1, \ldots, n, \nonumber
\end{align}
where 
$\{ C_i \}_{i = 1}^n$ 
denotes regularization parameters.
This formulation reduces to the standard formulation of the SVM when 
 all $C_i$'s are the same.
Our discussion in this paper holds for arbitrary choice of 
$C_i$'s.

We formulate the dual problem of (\ref{eq:svm.primal}) as:
\begin{align}
\begin{aligned}
  \max_{ \valpha}  & \ 
  - \frac{1}{2} \valpha^\top \vQ \valpha
   + \mb{1}^\top \valpha  \\ 
 {\rm s.t. } & \
   \vy^\top \valpha = 0, \ 
   \mb{0} \leq \valpha \leq \mb{c},
\end{aligned}
 \label{eq:svm.dual}
\end{align}
where 
$\valpha = [\alpha_1, \ldots, \alpha_n]^\top$,
$\vc = [C_1, \ldots, C_n]^\top$
 and $(i,j)$ element of
$\vQ \in \RR^{n \times n}$
is
$Q_{ij} = y_i y_j \Phi(\vx_i)^\top \Phi(\vx_j)$.
Note that, we use 
 inequalities between vectors as the element-wise inequality
(i.e., 
$\valpha \leq \vc$ 
$\Leftrightarrow$
$\alpha_i \leq C_i$ for $i = 1, \ldots, n$
).
Using kernel function 
$K(\vx_i, \vx_j) = \Phi(\vx_i)^\top \Phi(\vx_j)$, discriminant function
$f$ is represented as:
\begin{eqnarray*}
f(\vx) = \sum_{i = 1}^n \alpha_i y_i K(\vx, \vx_i) + \alpha_0.
\end{eqnarray*}

In what follows, the subscript by an index set such as 
$\mb{v}_\cI$ for a vector 
$\mb{v} = [v_1, \cdots, v_n]^\top$
indicates a
sub-vector of $\mb{v}$ whose elements are indexed by 
$\cI = \{ i_1, \ldots, i_{|\cI|}\}$.
For example, for 
$\mb{v} = [a, b, c]^\top$ 
and 
$\cI = \{1,3\}$, 
$\mb{v}_{\cal I}=[a, c]^\top$.
 Similarly, the subscript by two index sets such as 
$\vM_{\cI_1, \cI_2}$
for a matrix 
$\vM \in \RR^{n \times n}$
denotes a sub-matrix whose rows and columns are indexed by 
$\cI_1$
and 
$\cI_2$, respectively.
The principal sub-matrix such as 
$\vM_{\cI, \cI}$
is abbreviated as $\vM_\cI$.

\subsection{Solution Path Algorithm for SVM}

In this paper, we consider the solution path with respect to the
regularization
parameter vector $\vc$.
To follow the path, 
we parametrized $\vc$ in the following form:
\begin{eqnarray*}
 \vc^{(\theta)} = \vc^{(0)} + \theta \vd,
\end{eqnarray*}
where 
$\mb{c}^{(0)} = [C^{(0)}_1, \ldots, C^{(0)}_n]^\top$
is some initial parameter,
$\mb{d} = [d_1, \ldots, d_n]^\top$
is a direction of the path and 
$\theta \geq 0$.
We trace the change of the optimal solution of the SVM when 
$\theta$ increases from $0$.

Let 
$\{ \alpha_i^{(\theta)} \}_{i = 0}^n$ 
be the optimal parameters
and 
$\{ f_i^{(\theta)} \}_{i = 1}^n$ 
be the outputs $f(\vx_i)$ at $\theta$.
The KKT optimality conditions are summarized as:
\begin{subequations}
 \begin{eqnarray} 
  y_i f^{(\theta)}_i \geq 1, & {\rm if} &  \alpha^{(\theta)}_i = 0, 
  \label{eq:yf.gt.1} \\
  y_i f^{(\theta)}_i = 1, & {\rm if} &  
   0 < \alpha^{(\theta)}_i < C_i^{(\theta)},
  \label{eq:yf.eq.1} \\
  y_i f^{(\theta)}_i \leq 1, & {\rm if} &  \alpha^{(\theta)}_i = C^{(\theta)}_i,
  \label{eq:yf.lt.1} \\
   \vy^\top \valpha = 0. &&
  \label{eq:const.eq.dual}
 \end{eqnarray}
\label{eq:kkt}
\end{subequations}
We separate data points into three index sets 
$\cM, \cO, \cI \subseteq \{ 1, \ldots, n \}$
in such a way that these sets 
satisfy
\begin{subequations}
 \begin{eqnarray}
  i \in \cO &\Rightarrow&
  y_i f^{(\theta)}_i \geq 1, \alpha^{(\theta)}_i = 0, 
  \label{eq:idx.outside}  \\
  i \in \cM &\Rightarrow&
  y_i f^{(\theta)}_i = 1, \alpha^{(\theta)}_i \in [0, C_i],
  \label{eq:idx.margin}  \\
  i \in \cI &\Rightarrow&
  y_i f^{(\theta)}_i \leq 1, \alpha^{(\theta)}_i = C_i,
  \label{eq:idx.inside} 
 \end{eqnarray}
 \label{eq:idx.sets}
\end{subequations}
and we denote these partitions altogether as $\pi := (\cO, \cM, \cI)$.
If every data point belongs to one of the three index sets and
equality (\ref{eq:const.eq.dual}) holds, 
the KKT conditions (\ref{eq:kkt}) are satisfied.
As long as these index sets are unchanged, 
we have analytical expression of the optimal solution in the form
of 
$\alpha_i^{(\theta + \Delta \theta)} = \alpha_i^{(\theta)} + \Delta \theta \beta_i$, 
$i = 0, \ldots, n$,
where $\Delta \theta$ is the change of $\theta$ and
$\{ \beta_i \}_{i = 0}^n$ 
are constants derived from sensitivity
analysis theory:

\begin{thm} \label{thm:}
Let 
$\pi = (\cO, \cM, \cI)$ 
be the partition at the optimal solution at 
$\theta$
and assume that
\begin{eqnarray*}
\vM =
  \begin{bmatrix}
   0 & \vy_{\cM}^\top \\
   \vy_{\cM} & \vQ_{\cM} 
  \end{bmatrix}
\end{eqnarray*}
is non-singular\footnote{%
The invertibility of the matrix $\vM$ is assured if and only if 
the submatrix $\vQ_\cM$ is positive definite in subspace
$\{\vz \in \RR^{|\cM|} \mid \vy_\cM^\top \vz = 0 \}$.}.
Then, as long as $\pi$ is unchanged,
$\{ \beta_i \}_{i = 0}^n$ is given by
\begin{align}
 \begin{bmatrix}
  \beta_0 \\
  \vbeta_\cM
 \end{bmatrix} 
\! = \!
 - \vM^{-1}
 \begin{bmatrix}
  \vy_{\cal I}^\top \\
  \vQ_{{\cal M},{\cal I}} 
 \end{bmatrix}
 \mb{d}_\cI, \  
 \vbeta_\cO = \bm{0}, \  
 \vbeta_\cI = \mb{d}_\cI.  
\label{eq:beta}
\end{align}
\end{thm}
The proof is in \red{Appendix \ref{sec:proof_thm1}}.
This theorem can be viewed as one of the specific forms of the
sensitivity theorem \cite{Fiacco76}.
It can be derived from the KKT conditions 
(\ref{eq:kkt}) and the similar properties are repeatedly used in 
various solution path algorithms in machine learning
\citep{Cauwenberghs01, Hastie04}.

Using the above theorem, we can update the solution by 
$\alpha_i^{(\theta + \Delta \theta)} = \alpha_i^{(\theta)} + \Delta \theta \beta_i$
as long as $\pi$ is unchanged.
However, if we changes $\theta$, the optimal partition $\pi$ could also
 changes.
Those change points are called \emph{breakpoints}.
In the solution path algorithm, the optimality conditions 
are always kept satisfied by precisely
 detecting the breakpoints and updating $\pi$ properly.

\section{Suboptimal Solution Path}
\label{sec:subopt_path}

In this section, we develop a suboptimal solution path algorithm for the
SVM, where the tolerance to the optimality conditions can be arbitrary
controlled by users.
The basic idea is to relax the KKT optimality conditions and allow
multiple data points to move among the partition $\pi$ at the same time.
Note that it reduces the number of breakpoints and leads to the
improvement in its computational efficiency: 
allowing us to control the balance between the accuracy of the solution
and the computational cost.

\subsection{Approximate Optimality Conditions}

First, we relax the conditions (\ref{eq:idx.sets}) as
\begin{subequations}
 \begin{align}
  i \in \cO & \Rightarrow 
  y_i f^{(\theta)}_i \geq 1 - \varepsilon_1, 
  \alpha^{(\theta)}_i \in [-\varepsilon_2, 0], 
  \label{eq:app.idx.outside} \\
  i \in \cM & \Rightarrow   
  y_i f^{(\theta)}_i \in [1 \! - \! \varepsilon_1, 1 \! + \! \varepsilon_1], 
  \alpha^{(\theta)}_i \in [- \varepsilon_2, C_i^{(\theta)} \! + \! \varepsilon_2],
  \label{eq:app.idx.margin} \\
  i \in \cI & \Rightarrow 
  y_i f^{(\theta)}_i \leq 1 \! + \! \varepsilon_1, 
  \alpha^{(\theta)}_i \in [C^{(\theta)}_i, C^{(\theta)}_i \! + \! \varepsilon_2],
  \label{eq:app.idx.inside} 
 \end{align}
  \label{eq:app.idx.sets}
\end{subequations}
where 
 $\varepsilon_1 \geq 0$ and $\varepsilon_2 \geq 0$ specify the degree of
 approximation.
If we set $\varepsilon_1 = \varepsilon_2 = 0$, these conditions reduce
 to (\ref{eq:idx.sets}).

Our algorithm changes $\theta$ while keeping the above conditions 
(\ref{eq:app.idx.sets}) satisfied.
Let $\theta_0 = 0$ be the initial value of $\theta$ and
the non-decreasing sequence 
$\theta_0 \leq \theta_1 \leq \theta_2 \leq \ldots$,
be the breakpoints.
Suppose we are currently at $\theta_k$, the next breakpoint 
$\theta_{k + 1}$ is characterized as the point that we can not increase 
$\theta$ without violating the conditions (\ref{eq:app.idx.sets})
or changing index sets $\pi$.

If we set $\{ \beta_i \}_{i = 0}^n$ by (\ref{eq:beta}), then
$y_i f_i^{(\theta)}$, 
$i \in \cM$,
and 
$\alpha^{(\theta)}_i$, 
$i \in \cO \cup \cI$,
are constants.
To increase $\theta$ from $\theta_k$,
we only need to check the following inequalities:
\begin{align*}
 \begin{array}{rcll}
 y_i f_i^{(\theta_k)} + \Delta \theta g_i &\geq& 1 - \varepsilon_1, &
  \ i \in \cO, \\
 \alpha_i^{(\theta_k)} + \Delta \theta \beta_i 
  &\in& [-\varepsilon_2, C_i^{(\theta_k)} + \varepsilon_2], &
  \ i \in \cM, \\
 y_i f_i^{(\theta_k)} + \Delta \theta g_i &\leq& 1 - \varepsilon_1, &
  \ i \in \cI,
 \end{array}
\end{align*}
where $g_i$ is 
the change of output $y_i f_i$ which is 
defined by
 $\vg = \vQ \vbeta + \vy \beta_0$.
We want to know the maximum $\Delta \theta$ which satisfies all of the above inequalities.
We can easily calculate the maximum $\Delta \theta$ for each
inequality as follows:
\begin{eqnarray*}
\begin{array}{rcl}
 \Theta_{\cO} &=& 
  \Bigl\{
   (1 - \varepsilon_1 - y_i f_i^{(\theta_k)}) / g_i
   \Big|
   i \in \cO, g_i < 0 
  \Bigr\}, \\
 \Theta_{\cM_\ell} &=& 
  \left\{
  \left.
   - (\alpha_i^{(\theta_k)} + \varepsilon_2) / \beta_i
  \right|
  i \in \cM, \beta_i < 0
 \right\}, \\
 \Theta_{\cM_u} &=& 
  \Bigl\{
   (C^{(\theta_k)}_i + \varepsilon_2 - \alpha_i^{(\theta_k)}) / (\beta_i - d_i) 
   \\  && \hspace{8em} 
   \Big|  i \in \cM, \beta_i > d_i
   \Bigr\}, \\
 \Theta_{\cI} &=& 
  \Bigl\{
   (1 + \varepsilon_1 - y_i f_i^{(\theta_k)}) / g_i
  \Big|
  i \in \cI, g_i > 0 
 \Bigr\}, 
\end{array}
\end{eqnarray*}
Since we have to keep all of the inequalities satisfied,
we take the minimum of these values:
$\Delta \theta = \min \Theta$,
where 
$\Theta = \{ \Theta_\cO, \Theta_{\cM_\ell}, \Theta_{\cM_u}, \Theta_\cI \}$.
Then we can find $\theta_{k + 1} = \theta_k + \Delta \theta$.

Although we detect $\theta_{k + 1}$,
it is necessary to update $\pi$ to go beyond the breakpoint.
Conventional solution path algorithms 
allow only one data point to move
between
the partition $\pi$ at each breakpoint.
For example, $\alpha_i$, $i \in \cM$, reaches $0$, 
the algorithm
transfers
the index $i$ 
from $\cM$ to $\cO$
(\figurename~\ref{fig:bp.exact}).
In our algorithm, multiple data points are allowed to move between 
the partitions $\pi$ at the same time in order to reduce the number of breakpoints.

\subsection{Update Index Sets}

\begin{figure}[t]
\begin{center}
 \subfigure[Exact path]{
 \includegraphics[width=1.5in]{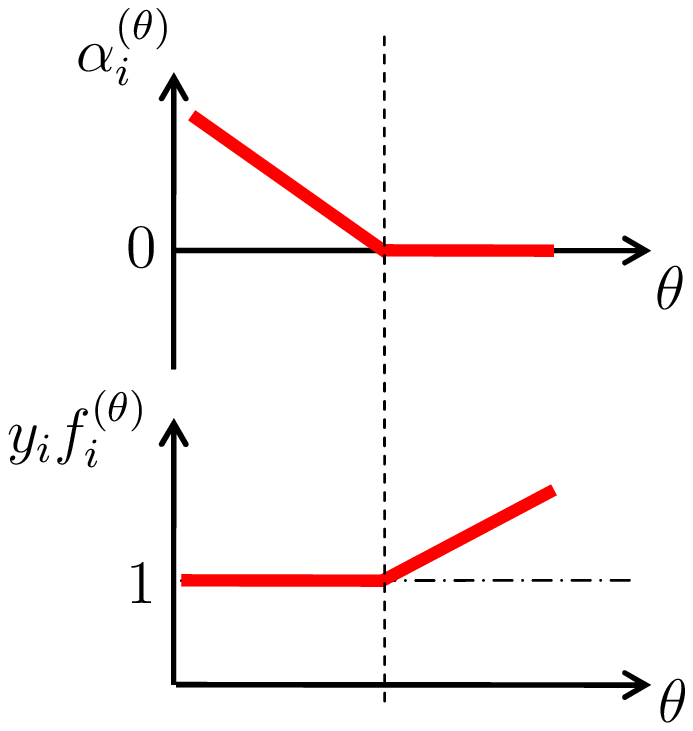}
 \label{fig:bp.exact}
 }
 \subfigure[Suboptimal path]{
 \includegraphics[width=1.5in]{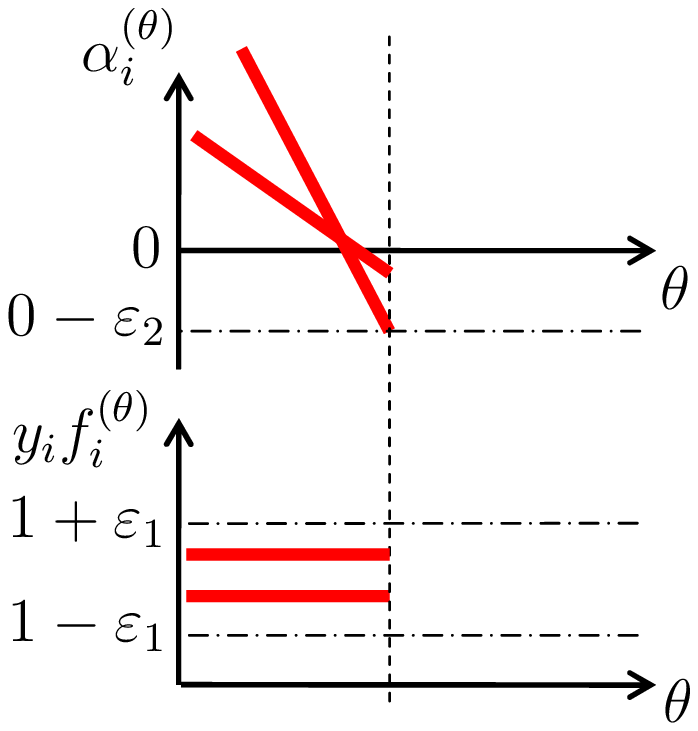}
 \label{fig:bp.app}
 }
 \caption{%
 An illustrative example of the breakpoint. 
 The points of the vertical dashed lines are breakpoints.
 (a)
 At the breakpoint in the upper plot, 
 $\alpha_i, i \in \cM,$ reaches $0$.
 Since the index 
 $i$ is transferred from $\cM$ to $\cO$,
 $\alpha_i = 0$ on the right side of the vertical line.
 In the lower plot, 
 $y_i f_i^{(\theta)} = 1$ 
 on the left side of the vertical line and
 $y_i f_i^{(\theta)} \geq 1$ on the right side of the vertical
 line.
 At the breakpoint, the data point $i$ satisfies the both of the
 optimality conditions (\ref{eq:idx.margin}) and (\ref{eq:idx.outside})
 for $\cM$ and $\cO$, respectively.
 (b)
 At the breakpoint in the upper plot,
 one of $\alpha_i, i \in \cM,$ reaches $- \varepsilon_2$.
 In the lower plot, both of the two lines are in 
 $[1-\varepsilon_1, 1+\varepsilon_1]$.
 In this case, 
 these two points 
 satisfy the both of the
 optimality conditions (\ref{eq:app.idx.margin}) and (\ref{eq:app.idx.outside})
 for $\cM$ and $\cO$, respectively.
 It does not necessarily mean that these two data points should move to
 $\cO$:
 either of them have a possibility to stay in $\cM$ even after the
 breakpoint.
 This situation is called degeneracy in parametric programming literature.
 }
\end{center}
\end{figure} 

At a breakpoint,
our algorithm handles all the data points that violate the strict
inequality conditions (\ref{eq:idx.sets}) rather than the relaxed ones 
(\ref{eq:app.idx.sets}) (\figurename~\ref{fig:bp.app}).
This situation can be interpreted as
what is called \emph{degeneracy} in the
parametric programming \citep{Ritter84}.
Here, degeneracy means that multiple constraints 
hit their
boundaries of inequalities simultaneously.
Although degenerate situation rarely happens in conventional solution
path algorithms, it is not the case in ours.
The simultaneous change of multiple data points inevitably brings about 
``highly'' degenerate situations involved with many constraints.
In degenerate case, we have a problem called the \emph{cycling}.
For example, if we move 
two indices $i$ and $j$ from $\cM$ to $\cO$ at the breakpoint,
then both or either of them may immediately return to $\cM$.
To avoid the cycling, we need to design an update strategy for $\pi$
that can circumvent cycling.

The degeneracy can be handled by several approaches 
which are known in the parametric
programming literature.
\citet{Ritter84} showed that the cycling can be dealt with through the 
well-known Bland's minimum index rule in the linear programming
\citep{Bland77}.
However, in the worst case, 
this approach must go through all the possible patterns of next $\pi$.
Since we need to evaluate $\{\beta_i\}_{i = 0}^n$ 
in each iteration, a large number of iterations may cause additional
computational cost.
In this paper, we provide more essential solution to this problem based on 
\citep{Berkelaar97}.

Suppose we are currently on the breakpoint $\theta_k$.
Let 
\begin{eqnarray*}
 \begin{array}{rcl}
 \cB_\cO &=&
  \{ i \mid \alpha_i^{(\theta_k)} \leq 0, \beta_i < 0, i \in \cM  \} 
  \cup \\
  && \{ i \mid y_i f_i^{(\theta_k)} \leq 1, g_i < 0, i \in \cO  \}, \\
 \cB_\cI &=&
  \{ i \mid \alpha_i^{(\theta_k)} \geq C_i^{(\theta_k)}, \beta_i > d_i, i \in \cM  \} 
  \cup \\
  && \{ i \mid y_i f_i^{(\theta_k)} \geq 1, g_i > 0, i \in \cI  \}.
 \end{array}
\end{eqnarray*}
$\cB_\cO$ is the 
set of indices which satisfy the conditions 
(\ref{eq:app.idx.outside}) and (\ref{eq:app.idx.margin}) 
for
being the member of $\cM$ and $\cO$ simultaneously 
at $\theta_k$.
Similarly, indices in $\cB_\cI$ 
satisfy the conditions 
(\ref{eq:app.idx.margin}) and (\ref{eq:app.idx.inside}) 
for being the member 
of $\cM$ and
$\cI$ at $\theta_k$.
Moreover, let us define sum of these two sets as
\begin{eqnarray*}
\cB = \cB_\cO \cup \cB_\cI.
\end{eqnarray*}
Our task is to partition these indices to 
$\cO$, $\cM$ and $\cI$ correctly so that it does not cause the cycling.

In our formulation, due to the approximation by 
$\varepsilon_1$ and $\varepsilon_2$,
the cycling may not occur at $\Delta \theta = 0$ immediately.
For example, suppose that $i$ move to $\cM$ from $\cO$ and its parameter is
$\alpha_i = 0$.
In the next iteration, we need to check 
$\alpha_i + \Delta \theta \beta_i \geq - \varepsilon_2$.
If $\beta_i < 0$, then we obtain
$\Delta \theta \leq -\varepsilon_2/\beta_i > 0$.
Although it allows 
$\Delta \theta > 0$, 
the index $i$ may return
back to $\cO$.
This situation can also be considered as cycling.

Let 
$\pi_{k} = (\cO_{k}, \cM_{k}, \cI_{k})$ 
be $\pi$ in 
$[\theta_k, \theta_{k+1}]$.
At 
$\theta_{k+1}$,
if and only if
the cycling does not occur, it can be shown that
the following conditions hold:
\begin{subequations}
\begin{eqnarray}
  \beta_i \geq 0, g_i = 0, && \text{ for } i \in \cM_{k+1} \cap \cB_\cO, 
   \label{eq:to.M.from.O} \\
  \beta_i = 0, g_i \geq 0, && \text{ for } i \in \cO_{k+1} \cap \cB_\cO, 
   \label{eq:to.O.from.O} \\
  \beta_i \leq d_i, g_i = 0, && \text{ for } i \in \cM_{k+1} \cap \cB_\cI, 
   \label{eq:to.M.from.I} \\
  \beta_i = d_i, g_i \leq 0, && \text{ for } i \in \cI_{k+1} \cap \cB_\cI. 
   \label{eq:to.I.from.I} 
\end{eqnarray}
\label{eq:cycling} 
\end{subequations}
Although $\beta_i$ and $g_i$ are usually calculated using $\pi$,
our approach allows us to calculate $\beta_i$ and $g_i$ 
 without knowing $\pi$ so that they can satisfy the above
conditions.
If the gradient $\vbeta$, 
which is defined in (\ref{eq:beta}), 
satisfies the following conditions,
we can find the next partition $\pi_{k+1}$ to satisfy
(\ref{eq:cycling}).
The conditions are:
 \begin{align}
 \begin{aligned}
  g_i \beta_i = 0, \
  g_i \geq 0, \
  \beta_i \geq 0, 
  \ i \in \cB_\cO,  &
  \\
  g_i (d_i - \beta_i)  = 0, \
  g_i \leq 0,  \
  \beta_i \leq d_i,  
  \ i \in \cB_\cI,  &
  \\
 \end{aligned} 
 \label{eq:grad.cond}
 \end{align} 
If we know such $\vbeta$ and $\vg$,
using the following update rule, we can determine $\pi_{k+1}$ as: 
\begin{align}
\begin{array}{rl}
 \cM_k = &  \cM_{k+\frac{1}{2}}  \cup  
 \{ i \mid \beta_i > 0, g_i = 0, i \in \cB_\cO \} \\ 
 & \cup \ 
 \{ i \mid \beta_i < d_i, g_i = 0, i \in \cB_\cI \}, \\
 \cO_k = &  \cO_{k+\frac{1}{2}}  \cup  
 \{ i \mid \beta_i = 0, g_i \geq 0, i \in \cB_\cO \}, \\
 \cI_k = &  \cI_{k+\frac{1}{2}}  \cup  
 \{ i \mid \beta_i = d_i, g_i \leq 0, i \in \cB_\cI \},
\end{array}
 \label{eq:update.pi}
\end{align}
where
$\cO_{k+\frac{1}{2}} = \cO_{k} \setminus \cB$,
 $\cM_{k+\frac{1}{2}} = \cM_{k} \setminus \cB$ and
 $\cI_{k+\frac{1}{2}} = \cI_{k} \setminus \cB$.

\begin{rmk}
 By definition, the update rule (\ref{eq:update.pi}) guarantees that
 the non-cycling conditions (\ref{eq:cycling}) hold.
\end{rmk}



To use (\ref{eq:update.pi}),
we need
$\vbeta$ (\ref{eq:beta}) which satisfies (\ref{eq:grad.cond}).
The following theorem 
shows that it can be obtained from a quadratic programming
problem (QP): 

\begin{thm} \label{thm:}
Let 
$\whbeta_0$,
$\whvbeta$ and
$\whvg$ 
be the optimal solutions of the following QP problem:
\begin{align}
& 
\min_{\whbeta_0, \whvbeta, \whvg} 
  \sum_{i \in \cB_\cO} \whg_i \whbeta_i + 
  \sum_{i \in \cB_\cI} \whg_i (\whbeta_i - d_i) 
  \label{eq:partition.qp} \\
& {\rm s.t.}  \left\{ 
 \begin{aligned}
  & \whvg_{\cB_\cO} \geq \bm{0}, \ \whvbeta_{\cB_\cO} \geq \bm{0}, \ 
   \whvg_{\cB_\cI} \leq \bm{0}, \ \whvbeta_{\cB_\cI} \leq \mb{d}_\cI, \\
  & \whvg_{\cM_{k+\frac{1}{2}}} = \bm{0}, \
    \whvbeta_{\cO_{k+\frac{1}{2}}} = \bm{0}, \
    \whvbeta_{\cI_{k+\frac{1}{2}}} = \mb{d}_{\cI_{k+\frac{1}{2}}}, \nonumber \\
  & \vy^\top \whvbeta = 0, \ 
   \whvg = \vQ \whvbeta + \vy \whbeta_0, \nonumber 
 \end{aligned}
 \right. 
\end{align}
and $\pi$ is determined by (\ref{eq:update.pi}) using 
$\whvbeta$ and $\whvg$.
Then
$\whbeta_0$,
$\whvbeta$ and
$\whvg$ 
satisfy (\ref{eq:grad.cond}) and they are equal to the gradient 
$\beta_0$,
$\vbeta$ and
$\vg$, respectively.
\end{thm}

Although the detailed proof is in \red{Appendix}, 
we can provide clear interpretation of this optimization problem.
The objective function and inequality constraints corresponds to 
 (\ref{eq:grad.cond})
and the other constraints
correspond to the linear system (\ref{eq:beta}).
It can be shown that the optimal value of the objective function is $0$.
Given the non-negativity of each term in the objective,
we see that (\ref{eq:grad.cond}) holds
(see \red{Appendix \ref{sec:proof_thm2}} for detail).

The optimization problem (\ref{eq:partition.qp}) has 
$2n + 1$ variables and $2|\cB| + 2n + 1$ constraints.
However, we can reduce these sizes to $|\cB|$ variables and $2|\cB|$
constraints
by arranging the equality
constraints\footnote{%

In the case of $|\cM_{k+\frac{1}{2}}| = 0$, the reduced problem has 
$|\cB| + 1$ variables $2 |\cB| + 1$ constraints.}.
The detailed formulation of the reduced problem is in \red{Appendix \ref{sec:reduced_QP}}.
If the size of $|\cB|$ is large, 
it may take large computational cost to
solve (\ref{eq:partition.qp}).
To avoid this, we set the upper bound $B$ for the number of elements of
$\cB$. 
In the case of $|\cB| > B$, we choose 
 top $B$ elements from the original $\cB$ 
by increasing order of $\Theta$
as the elements of $\cB$.

\subsection{Algorithm and Computational Complexity}

Here, we summarize our algorithm and analyze its computational
complexity.
At the $k$-th breakpoint, 
our algorithm performs the following procedure:
\begin{description}

 \item[step1] Using $\pi_k$, calculate $\beta_0, \vbeta$ and $\vg$ by (\ref{eq:beta})

 \item[step2] Calculate the next breakpoint $\theta_{k+1}$ and update
	    $\alpha^{(\theta)}_0, \valpha^{(\theta)}, \vc^{(\theta)}$;

 \item[step3] Solve (\ref{eq:partition.qp}) and calculate $\pi_{k+1}$ by (\ref{eq:update.pi})

\end{description}

In step1, we need to solve the linear system (\ref{eq:beta}).
In conventional solution path algorithms, we can update it using
rank-one-update of an inverse matrix or a Cholesky factor from previous
iteration by $O(|\cM|^2)$ computations.
In our case, we need rank-$m$-update at each breakpoint, where 
$1 \leq m \leq B$.
When we set $B$ as some small constant,
the computational cost still remains 
$O(|\cM|^2)$.
Including the other processes in this step,
the computational cost becomes 
$O(n |\cM|)$.
In step2, 
given $\vbeta$ and $\vg$,
we can calculate all the possible step length
$\Theta$ by $O(n)$.
In step3, 
since the optimization problem (\ref{eq:partition.qp}) becomes convex
QP problem with $|\cB|$ variables,
it can be solved efficiently by some standard QP solvers
in the situation $|\cB|$ is relatively small compared to $n$.
When we set $B$ as some constant, 
the time for solving this optimization problem is then independent of 
$n$.

Put it all together, in the case of constant $B$,
the computational cost of each breakpoint is 
$O(n |\cM|)$.
This is the same as the conventional solution path algorithm.
However, 
as we will see later in experiments,
 our algorithm drastically reduces the number of breakpoints especially
 when we use large $\varepsilon_1$ and
$\varepsilon_2$.

\section{Analysis}
\label{sec:analysis}

In this section, we provide some theoretical analyses of our suboptimal solution
path.

\subsection{Interpretation as Perturbed Problem}

An interesting property of our approach is 
that the solutions 
always keep the optimality of an optimization problem 
which is slightly perturbed 
from the original one.
The following theorem gives the formulation of the 
perturbed problem:
\begin{thm} \label{thm:}
 Every solution 
 $\valpha^{(\theta)}$ 
 in the suboptimal solution path 
 is the optimal solution of the following optimization problem:
 \begin{eqnarray}
  \begin{split}
    \max_{ \valpha }  ~~&
   - \frac{1}{2} \valpha^\top \vQ \valpha
   + (\mb{1} + \mb{p})^\top \valpha  \\
   {\rm s.t. } ~~&
   \textstyle{
   \vy^\top \valpha = 0, \ 
   - \mb{q} \leq \valpha \leq \mb{c}^{(\theta)} + \mb{q}.}
  \end{split}
  \label{eq:prtb.dual}
 \end{eqnarray} 
 where perturbation parameters 
 $\vp, \vq \in \RR^n$
 are in
 $-\varepsilon_1 \mb{1} \leq \vp \leq \varepsilon_1 \mb{1}$
 and
 $\mb{0} \leq \vq \leq \varepsilon_2 \mb{1}$,
 respectively.
\end{thm}

\begin{proof}
 Let 
 $\vxi^+, \vxi^- \in \RR^n_+$ and
 $\kappa \in \RR$ be the Lagrange multipliers.
 The Lagrangian is
 \begin{eqnarray*}
 L &=& 
  \textstyle{
  - \frac{1}{2} \valpha^\top \vQ \valpha
  + (\mb{1} + \mb{p})^\top \valpha 
  } \\
 && 
  \textstyle{
  + (\valpha + \mb{q})^\top \mb{\xi}^- 
  + (\mb{c}^{(\theta)} + \mb{q} - \valpha)^\top \mb{\xi}^+
  + \kappa \vy^\top \valpha,
  }
\end{eqnarray*}
 and the KKT conditions are
\begin{subequations}
 \begin{eqnarray}
  \textstyle{ \pd{L}{\valpha} = 
  - \vQ \valpha + \mb{1} + \mb{p}
  + \mb{\xi}^- - \mb{\xi}^+
  + \kappa \vy = \mb{0}, }
  \label{eq:prtb.dL}\\
  \mb{\xi}^+, \mb{\xi}^- \geq 0, 
   \label{eq:prtb.sign.xi} \\
  \xi_i^- (\alpha_i + q_i) = 0, \ i = 1, \ldots, n, 
   \label{eq:prtb.compl1}\\
  \xi_i^+ (C_i^{(\theta)} + q_i - \alpha_i) = 0, \ i = 1, \ldots, n,
   \label{eq:prtb.compl2} \\
  - \mb{q} \leq \valpha \leq \mb{c}^{(\theta)} + \mb{q}. 
   \label{eq:prtb.box} \\ 
  \vy^\top \valpha = 0, 
   \label{eq:prtb.eq}  
 \end{eqnarray}
 \label{eq:prtb.kkt}
\end{subequations}
Substituting 
$\valpha = \valpha^{(\theta)}$ and $\kappa = - \alpha^{(\theta)}_0$,
$i$-th element of (\ref{eq:prtb.dL}) can be written as 
$y_i f_i^{(\theta)} = 1 + p_i + \xi_i^- - \xi_i^+$.
Considering this and the conditions of suboptimal solution 
$\valpha^{(\theta)}$ (\ref{eq:app.idx.sets}), 
 there exist
 $p_i \in [-\varepsilon_1, \varepsilon_1]$ and 
 $\xi^\pm_i$
 which satisfy 
 $\xi^+_i = \xi^-_i = 0$ for $i \in \cM$, 
 $\xi^+_i = 0, \  \xi^-_i \geq 0$, for $i \in \cO$ and 
 $\xi^+_i \geq 0, \  \xi^-_i = 0$, for $i \in \cI$.
These $\xi_i^{\pm}$'s satisfy the non-negativity constraint 
(\ref{eq:prtb.sign.xi}).

The complementary conditions 
(\ref{eq:prtb.compl1}) and (\ref{eq:prtb.compl2})
for $i \in \cM$ hold from $\xi_i^+ = \xi_i^- = 0$.
For $i \in \cO$, 
since $\xi_i^+ = 0$, 
we don't have to check (\ref{eq:prtb.compl2}).
In this case, 
if we set 
$q_i = -\alpha^{(\theta)}_i \in [0, \varepsilon_2]$,
then (\ref{eq:prtb.compl1}) holds.
It can be shown in a similar way that
(\ref{eq:prtb.compl1}) and (\ref{eq:prtb.compl2}) hold
for $i \in \cI$.

Our suboptimal solution path algorithm always satisfies
the equality constraint of the dual (\ref{eq:svm.dual})
and the box constraint (\ref{eq:prtb.box}) 
satisfied. 
Therefore, we see
(\ref{eq:prtb.kkt}) holds.
\end{proof}

The problem (\ref{eq:prtb.dual}) can be interpreted as the dual problem
of the following form of the SVM:
\begin{eqnarray}
 \min_{\mb{w}, \alpha_0} \  \frac{1}{2} \mb{w}^\top \mb{w} + 
  \sum_{i = 1}^n \ell(1 + p_i - y_i f_i), 
  \label{eq:prtb.primal}
\end{eqnarray}
where
\begin{eqnarray*}
 \ell(\xi_i) = 
\left\{
 \begin{array}{ll}
 (C_i^{(\theta)} + q_i) \xi_i, & \mbox{ for } \xi_i \geq 0, \\
  - q_i \xi_i, & \mbox{ for } \xi_i < 0, \\
  \end{array}
\right.
\end{eqnarray*}
is a loss function.
We see that the perturbations present in the loss term.

\subsection{Error Analysis}

We have shown that the solution of the suboptimal solution path can be
interpreted as the optimal solution of the perturbed problem (\ref{eq:prtb.primal}).
Here, we consider how close the optimal solution of the perturbed
problem 
to the solution of the original problem in terms of the optimal
objective value.

Let $D(\valpha)$ and
 $\wtD(\valpha)$
be the dual objective functions of the original 
optimization problem (\ref{eq:svm.dual}) and
the perturbed problem (\ref{eq:prtb.dual}), respectively.
From the affine lower bound of
 $\wtD(\valpha)$,
 we obtain
\begin{align*}
 \wtD(\valpha) \leq
  D(\valpha^*) + \vp^\top \valpha^* +
  ( - \vQ \valpha^* + \mb{1} + \vp)^\top
  (\valpha - \valpha^*),
\end{align*}
where $\valpha^*$ is the optimal solution of the original problem.
Let $\wtvalpha$ be the optimal
 solution of the perturbed problem.
Substituting $\valpha = \wtvalpha$ and adding
$\alpha^*_0 \vy^\top (\wtvalpha - \valpha^*) = 0$
to the right hand side,
we obtain
\begin{align}
 \wtD(\wtvalpha) -
  D(\valpha^*) \leq  \vp^\top \valpha^* + 
  ( \vxi^* + \vp)^\top (\wtvalpha - \valpha^*),
 \label{eq:dual.affine.bound}
\end{align}
where
$\vxi^* = - \vQ \valpha^* - \vy \alpha^*_0 + \mb{1}$.
Note that 
$\vxi^*_\cI \geq \mb{0}$,
$\vxi^*_\cM = \mb{0}$ and
$\vxi^*_\cO \leq \mb{0}$,
where $\cI$, $\cM$ and $\cO$ represent the optimal partition of the original
problem (\ref{eq:svm.dual}).
Here, we define
$\widetilde{\cI} = \{ i \mid \xi^*_i + p_i \geq 0, \ i \in \cI \}$,
$\widetilde{\cO} = \{ i \mid \xi^*_i + p_i \leq 0, \ i \in \cO \}$ and
$\widetilde{\cM} = \{1, \ldots, n \} \setminus (\widetilde{\cO} \cup \widetilde{\cI})$.
From the right hand side of (\ref{eq:dual.affine.bound}), we obtain
\begin{align*}
  \wtD(\wtvalpha) - & D(\valpha^*) 
  \leq 
 \textstyle{
 \sum_{i \in \cM \cup \cI} |p_i| \ C_i^{(\theta)} + }
  \\
& 
 \textstyle{
\sum_{i \in \widetilde{\cI} \cup \widetilde{\cO}} 
  | \xi_i^* + p_i | \ q_i 
  + \sum_{i \in \widetilde{\cM} } 
  | p_i | \ (C_i^{(\theta)} + q_i)}
\end{align*}
From the duality theorem, this also bounds  
 the difference of the primal objective value.
Comparing the original objective function (\ref{eq:svm.primal}),
this bound can be considered small when 
$p_i$ and $q_i$ is enough small compared to 
$\xi_i^*$ and $C_i$.
In this view point, this bound gives theoretical justification 
for our intuitive interpretation.
The bound for 
$D(\valpha^*) - \wtD(\wtvalpha)$ can be
also derived in the same manner.

\section{Experiments}
\label{sec:experiment}

In this section, 
we illustrate the empirical performance of the proposed approach 
compare to the conventional exact solution path algorithm.
Our task is to
trace the solution path from
$\vc^{(0)} = 10^{-1} / n \times \mb{1}$ to 
$\vc^{(1)} = 10^{6} / n \times \mb{1}$.
Since all the elements of $\vc^{(\theta)}$ takes the same value in this
case, we sometimes refer to this common value as $C^{(\theta)}$
(i.e., $\vc^{(\theta)} = C^{(\theta)} \times \mb{1}$).
The RBF kernel
$K(\vx_i, \vx_j) = \exp(-\gamma \| \vx_i - \vx_j \|_2^2)$
is used with $\gamma = 1/p$ where 
$p$ is the number of features.
To circumvent possible numerical instability in the solution path,
we add small positive constant $10^{-6}$ to the diagonals of the matrix
$\vQ$.

Let $e \geq 0$ be a parameter which controls the degree of
 approximations.
In this paper, using $e$,
we set $\varepsilon_1$ and $\varepsilon_2$ as
$\varepsilon_1 = e$ and 
$\varepsilon_2 = e \times C^{(\theta_k)}$, respectively,
where $\theta_k$ is the previous breakpoint.
We set $\varepsilon_2$ 
using relative scale to $C^{(\theta_k)}$.

\tablename~\ref{tab:data} lists the statistics of data sets.
These data sets are available from 
LIBSVM site \citep{Chang04} and UCI data repository \citep{Asuncion07}.
We randomly sampled $n$ data points from the original data set $10$
times (we set $n$ be approximately $80$\% of the original number of data
 points in the table).
The input $\vx$ 
of each data set is
linearly scaled to 
$[0, 1]^p$.

\begin{table}[t]
\caption{Data set}
\label{tab:data}
\vskip 0.1in
\begin{center}
\begin{small}
\begin{tabular}{l|cc}
Data set & $n$ & $p$ \\
\hline
internet ad  & 2359 & 1558 \\
spam & 4601 & 57 \\
a5a  & 6414 & 123 \\
w5a  & 9888 & 300 \\
\end{tabular}
\end{small}
\end{center}
\vskip -0.1in
\end{table}

\figurename~\ref{fig:time.bp.vs.eps} shows 
the comparison of the CPU time
and the number of breakpoints.
To make fair comparison, 
the initialization is not included in
the CPU time.
In these results, we set $B = 10$ and we investigated the relationship 
between the computational cost and the degree of approximation by
examining several settings of 
$e \in \{ 0.001, 0.01, 0.1, 0.5 \}$.
The results indicate that our approach can reduce the CPU time
especially when $e$ is large.
The number of breakpoints were also reduced, in the same way as the CPU
time.
In our approach, since we need rank-$m$-update of matrix in each
breakpoint ($1 \leq m \leq B$), 
an update in a breakpoint may take longer time than 
rank-one-update which is needed in the conventional
solution path algorithm. 
We conjecture that this is why the decrease in the number of breakpoints
was slightly faster than the CPU time.
However, since the maximum value of $|\cB|$ was set as $B = 10$ in this experiment,
this additional cost was relatively small compared to the effect of the
reduction of the number of breakpoints.

\begin{figure}[tb]
 \begin{center}
 \subfigure[ad]{
 \includegraphics[width=2in]{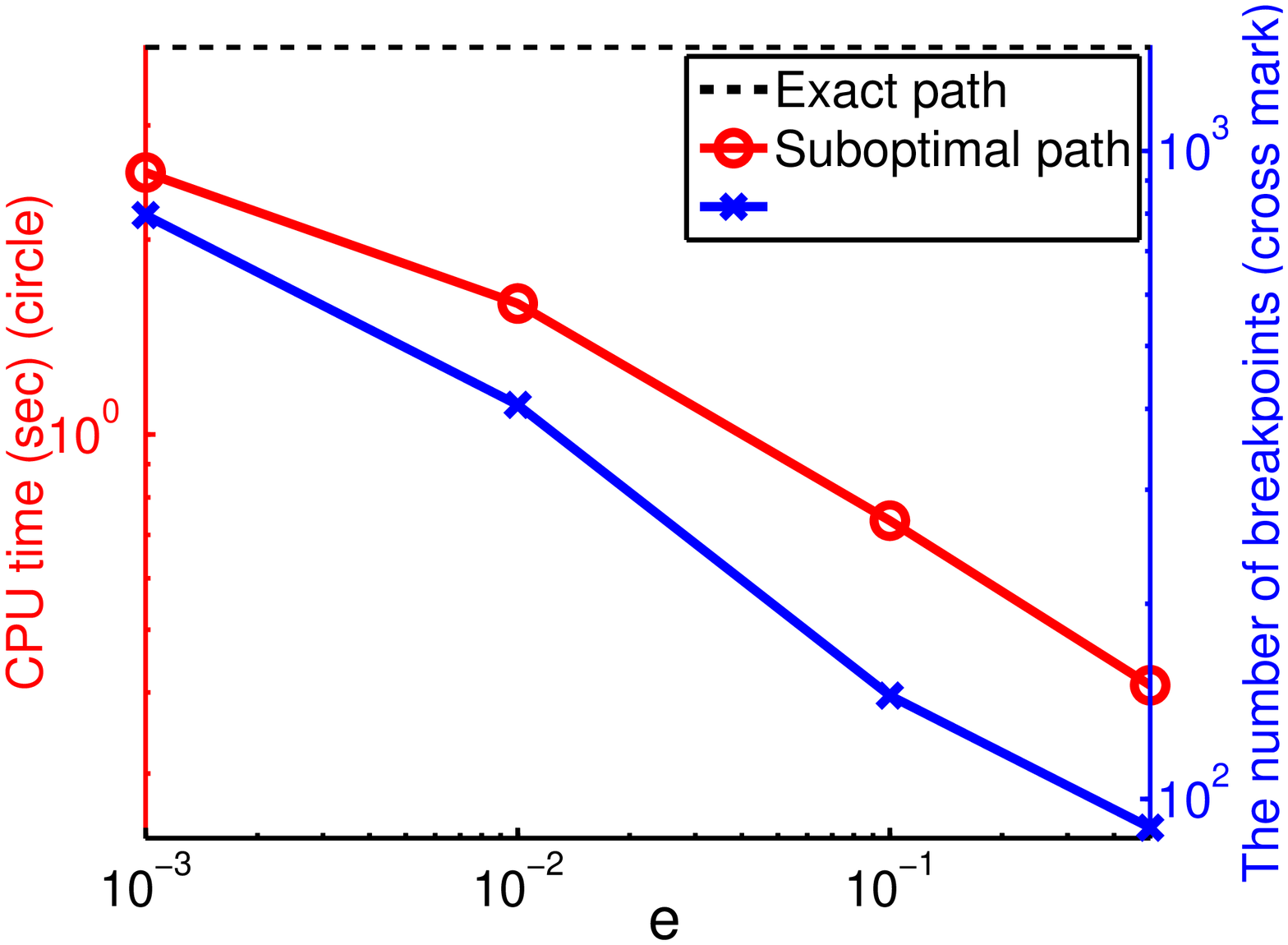}
 }
 \subfigure[spam]{
 \includegraphics[width=2in]{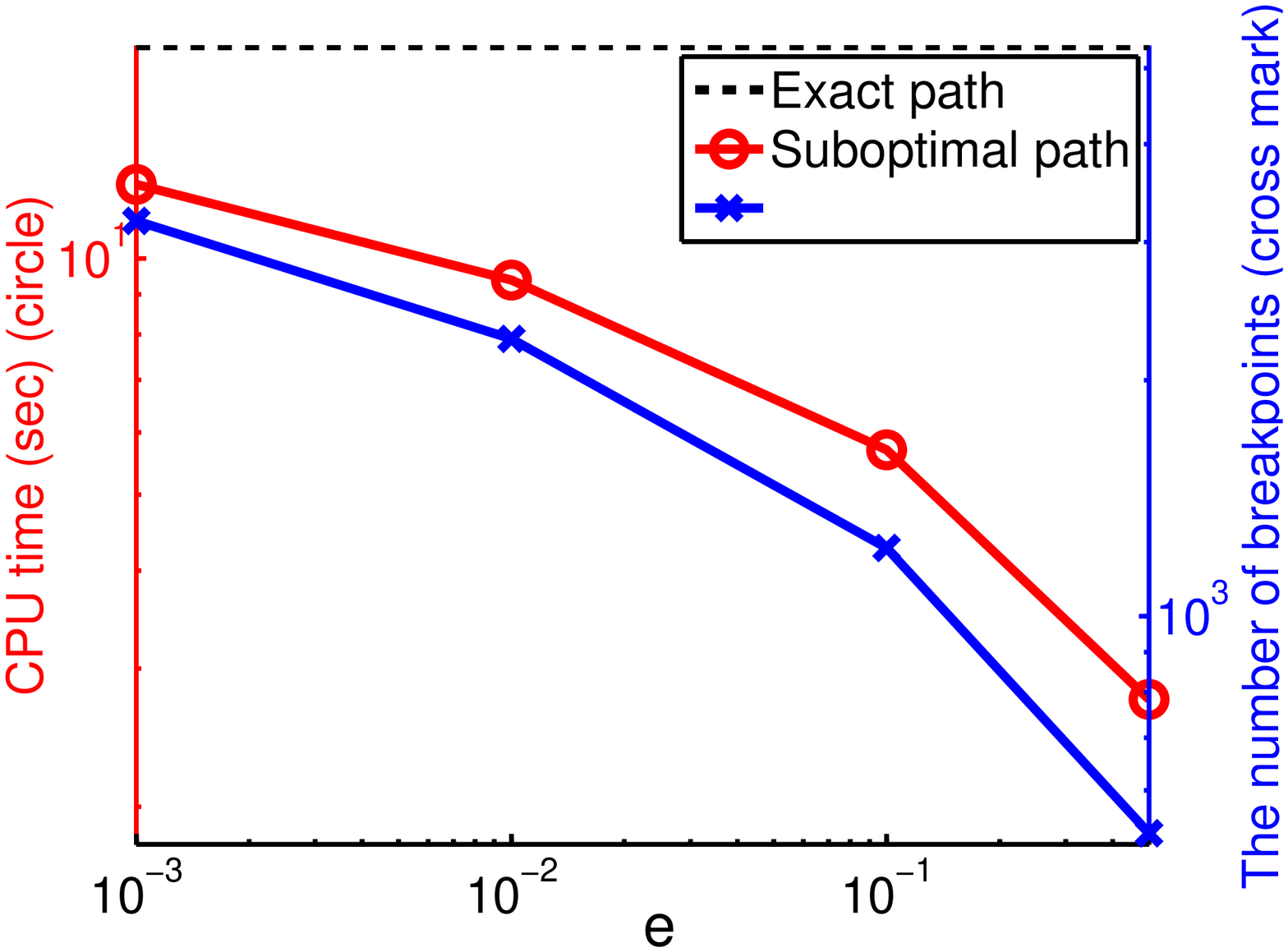}
 }
 \subfigure[a5a]{
 \includegraphics[width=2in]{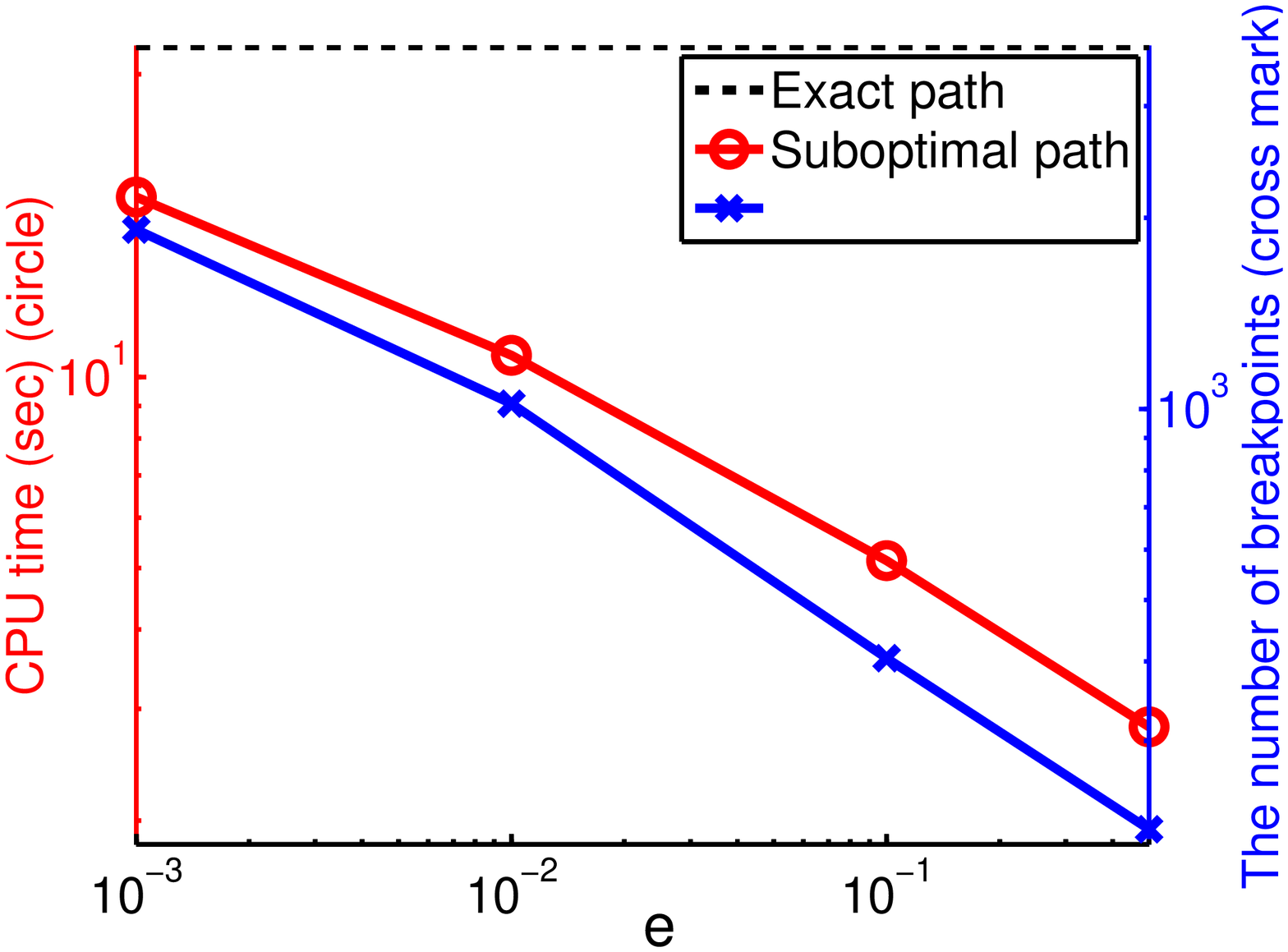}
 }
 \subfigure[w5a]{
 \includegraphics[width=2in]{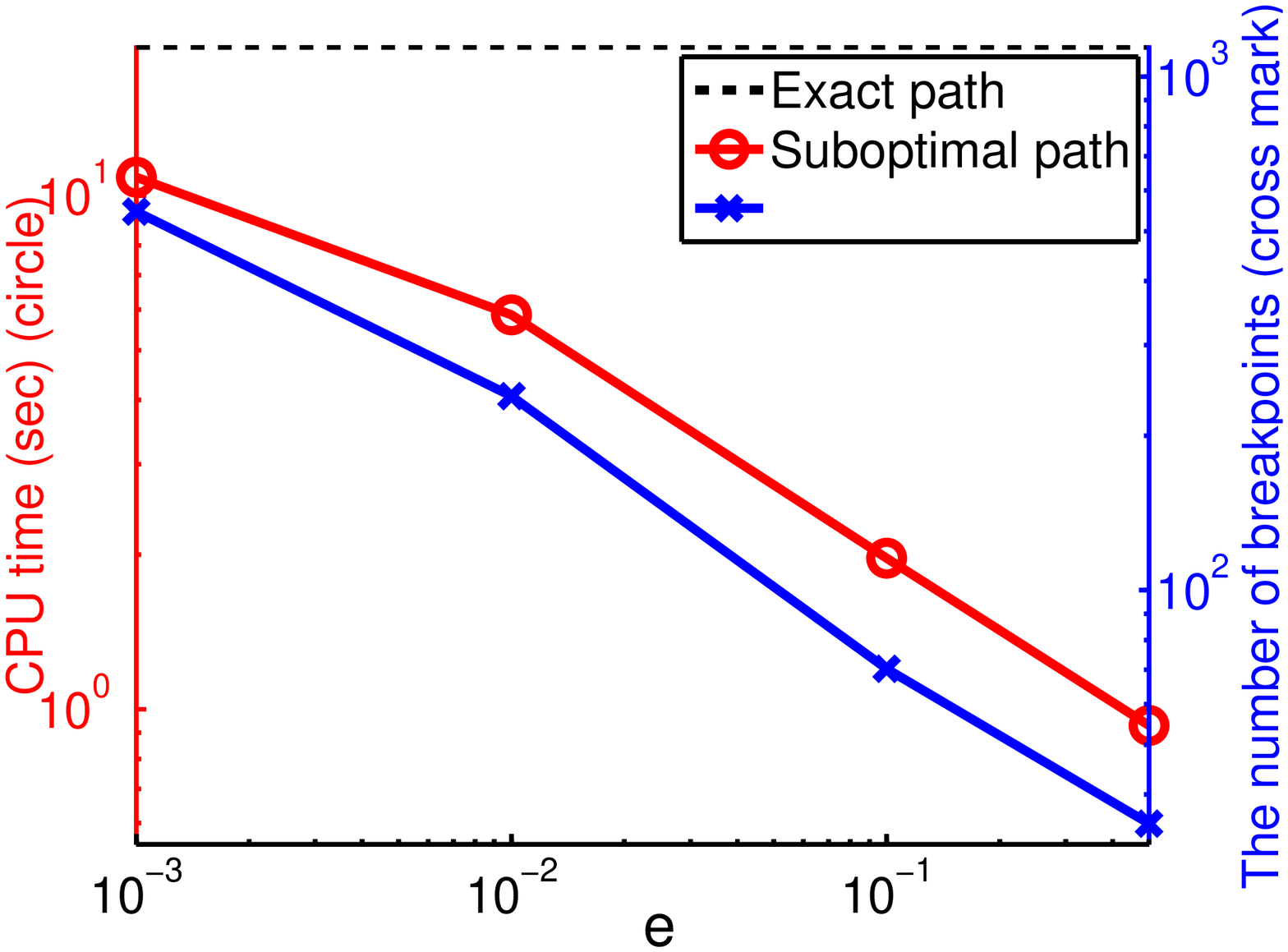}
 }
 \caption{%
 Log plot of CPU time and the number of breakpoints. 
 The horizontal axis of each plot is the degree of the
 approximation.
 The circle denotes the CPU time (left axis) and the cross mark denotes
 the number of breakpoints (right axis) of the suboptimal path.
 The top dashed line of each plot means both of the CPU time and the
 number of breakpoints of the exact path.
 The relative scale of the left and right axes are the same.
 }
 \label{fig:time.bp.vs.eps}
 \end{center}
\end{figure}


Next, we investigated the effect of $B$.
\figurename~\ref{fig:time.bp.vs.B} shows the CPU time and the number of
breakpoints for w1a data ($n = 2477$, $p = 300$) with $B = 10$ and $B = n$.
When $B = n$, there are no upper bounds for $|\cB|$.
In the left plot, when $B= n$, 
we see that the CPU time 
is longer than the case of $B = 10$.
In this data set,
this difference of the CPU time mainly comes from the cost of 
 the matrix update and
QP (\ref{eq:partition.qp})
whose size is
proportional to $|\cB|$ (data not shown).
On the other hand, in the left plot, 
the number of breakpoints is stable in the both case of 
$B = n$ and $B = 10$, and interestingly,
the number itself is almost the same in these two settings.
Our results suggest that too many $\cB$ does not 
contribute to reduce the number of breakpoint.
Although
these unstable results in $B = n$ is not always happen, 
we observed that it is more stable to use
 $B = 10$ or $B = 100$ in several other data sets.

\begin{figure}[tb]
 \begin{center}  
 \subfigure[CPU time]{
 \includegraphics[width=2in]{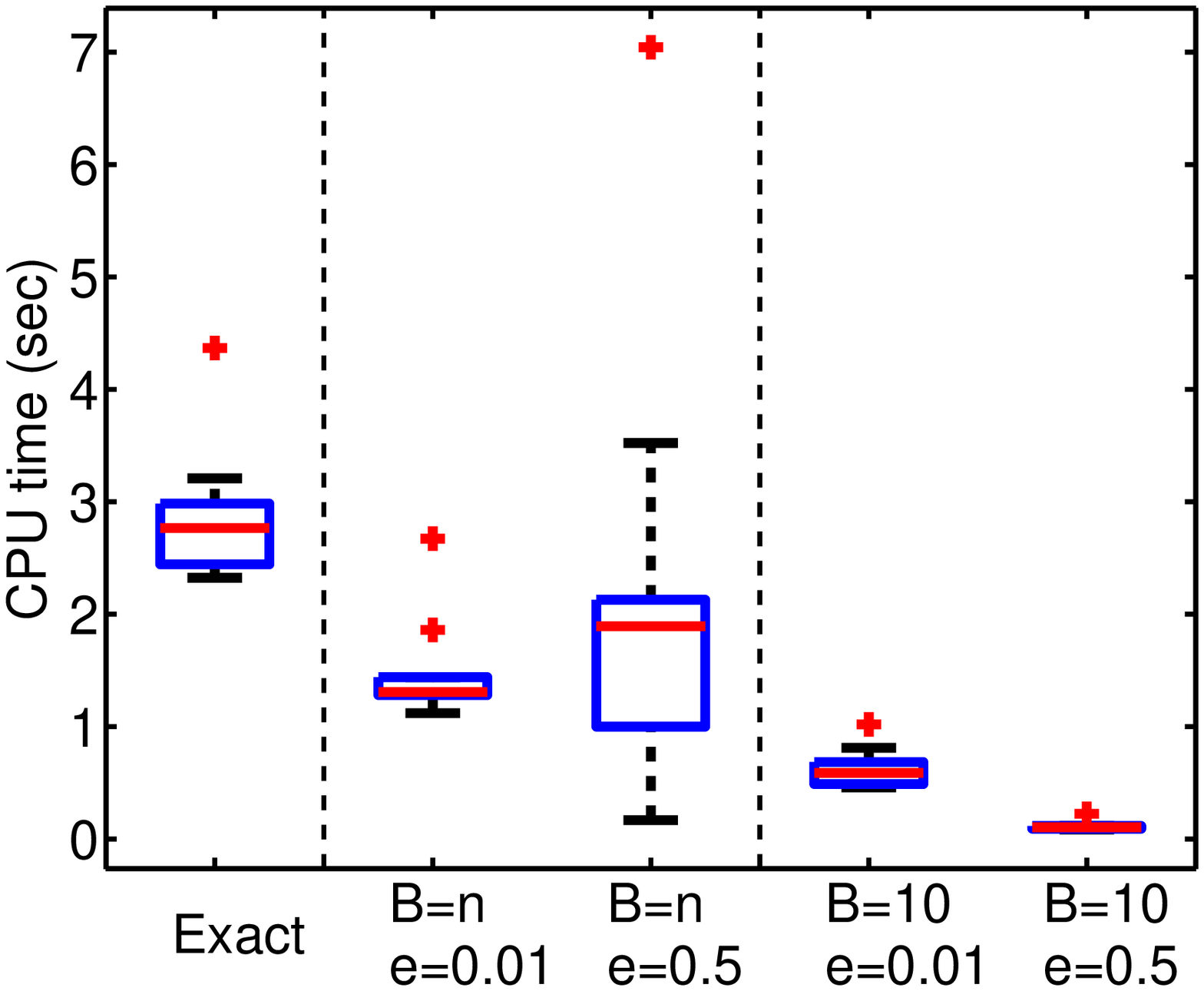}
 }
 \subfigure[Breakpoint]{
 \includegraphics[width=2in]{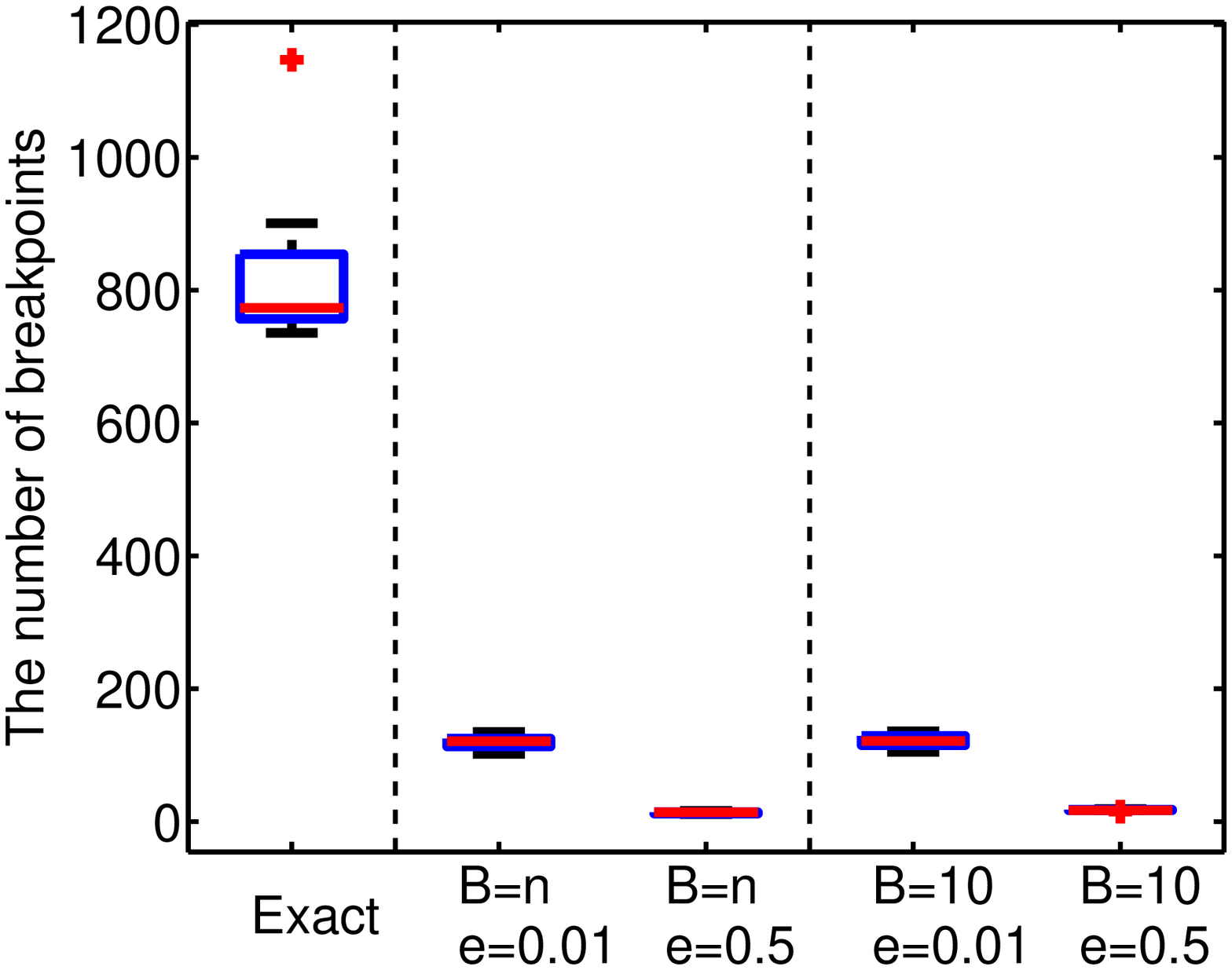}
 }
 \caption{The comparisons for different settings of $B$.}
 \label{fig:time.bp.vs.B}
 \end{center}
\end{figure}

We also compared the difference of $\pi$ between 
the exact solution path and the suboptimal path
in order to see the degree of approximation in terms of the active set.
Let $I_i \in \{0,1\}$ be an indicator variable which has $1$
when a data point $i$ belongs to different set
among $\cM$, $\cO$ and $\cI$ between two solution paths.
\figurename~\ref{fig:seterr} shows plots of $10$ runs average of
$\sum_{i=1}^n I_i / n$ for 
$e =  0.5$ in a5a data set.
We see that
the difference is at most about $10$\%.
\figurename~\ref{fig:idxptrn} shows the size of each index set
(this plot is one of $10$ runs).
Although the small differences exist, 
the changing patterns are similar each other.

\begin{figure}[tb]
 \begin{center}  
\subfigure[The difference of $\pi$]{
\includegraphics[width=2in]{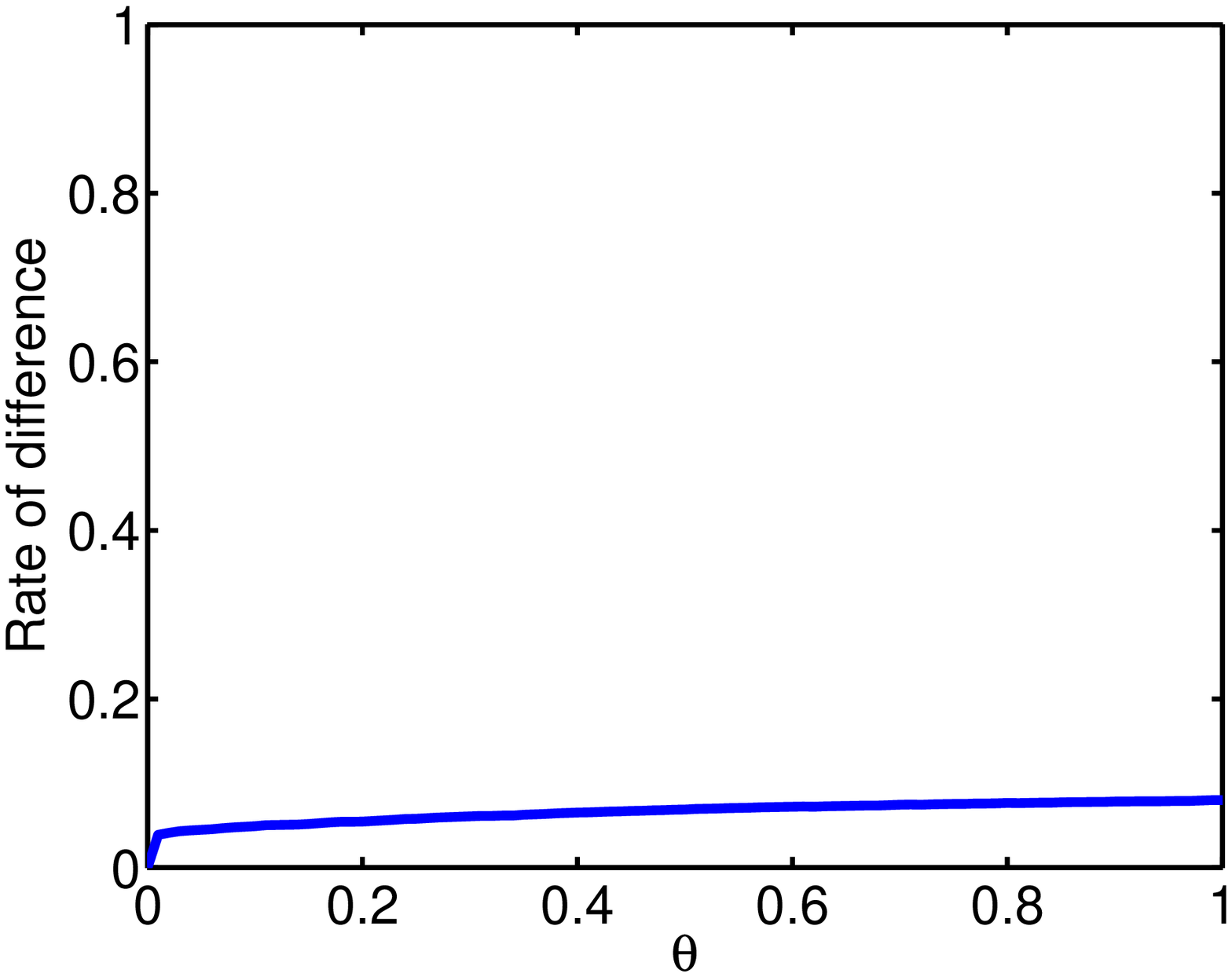}
\label{fig:seterr}
}
%
\subfigure[The size of each index set (solid: exact path, dashed: suboptimal path)]{
 \includegraphics[width=2in]{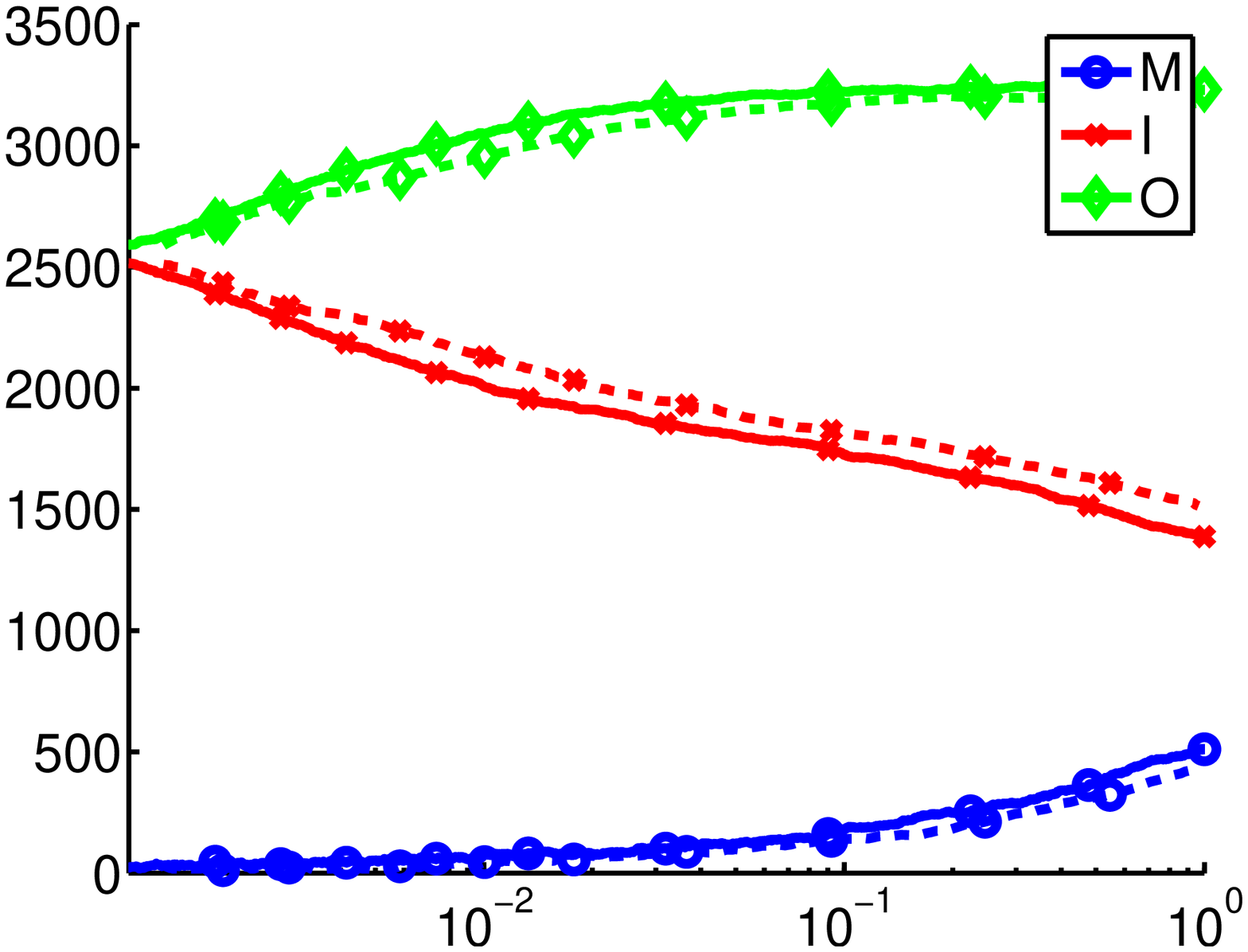}
 \label{fig:idxptrn}
 }
 \caption{Comparisons of the behavior of $\pi$}
 \end{center}
\end{figure}

%

\tablename~\ref{tab:terr}
shows results of test error rate comparison for $e = 0.5$.
We used $60$\% of the data for training, $20$\% for validation and
$20$\% for testing.
In each data set, 
we see that the performances of our suboptimal solutions 
are comparable to the exact solution path.


\begin{table}
  \caption{Test error rate and its standard error}
  \label{tab:terr}
  \vskip 0.1in
 \begin{center} 
 \begin{tabular}{l|cc}
  data & exact path & $e = 0.5$  \\ \hline
  ad & 0.0326 (0.0021) & 0.0328 (0.0026) \\ 
  spam & 0.0770 (0.0036) & 0.0812 (0.0037) \\ 
  a5a & 0.1587 (0.0025) & 0.1597 (0.0031) \\ 
  w5a & 0.0171 (0.0012) & 0.0176 (0.0010) \\ 
 \end{tabular}
 \end{center}
\vskip 0.1in
\end{table}



\section{Conclusion}

In this paper, we have developed a suboptimal solution path algorithm
 which
 traces the changes of solutions under the relaxed
optimality conditions.
Our algorithm can reduce the number of breakpoints by moving multiple
indices in $\pi$ at one breakpoint.
Another interesting property of our approach is that
 the suboptimal solutions exactly correspond to the optimal solutions of
the perturbed problems from the original SVM optimization problems.
The experimental results demonstrate that our algorithm efficiently
follows the path and it has similar patterns of active sets and classification
 performances compared to the exact path.

\small{
\bibliography{ref}

\begin{thebibliography}{14}
\providecommand{\natexlab}[1]{#1}
\providecommand{\url}[1]{\texttt{#1}}
\expandafter\ifx\csname urlstyle\endcsname\relax
  \providecommand{\doi}[1]{doi: #1}\else
  \providecommand{\doi}{doi: \begingroup \urlstyle{rm}\Url}\fi

\bibitem[Asuncion \& Newman(2007)Asuncion and Newman]{Asuncion07}
Asuncion, A. and Newman, D.~J.
\newblock {UCI} machine learning repository.
\newblock \url{http://www.ics.uci.edu/~mlearn/MLRepository.html}, 2007.

\bibitem[Berkelaar et~al.(1997)Berkelaar, Roos, and Terl{\'a}ky]{Berkelaar97}
Berkelaar, A.~B., Roos, K., and Terl{\'a}ky, T.
\newblock The optimal set and optimal partition approach to linear and
  quadratic programming.
\newblock In Greenberg, H. and Gal, T. (eds.), \emph{Advances in Sensitivity
  Analysis and Parametric Programming}, chapter~6. Kluwer Academic Publishers,
  1997.

\bibitem[Best(1982)]{UWaterloo:Best:1982}
Best, M.~J.
\newblock An algorithm for the solution of the parametric quadratic programming
  problem.
\newblock Technical Report 82-24, Faculty of Mathematics, University of
  Waterloo, 1982.

\bibitem[Bland(1977)]{Bland77}
Bland, R.~G.
\newblock New finite pivoting rules for the simplex method.
\newblock \emph{Mathematics of Operations Research}, 2:\penalty0 103--107,
  1977.

\bibitem[Cauwenberghs \& Poggio(2001)Cauwenberghs and Poggio]{Cauwenberghs01}
Cauwenberghs, G. and Poggio, T.
\newblock Incremental and decremental support vector machine learning.
\newblock In Leen, Todd~K., Dietterich, Thomas~G., and Tresp, Volker (eds.),
  \emph{Advances in Neural Information Processing Systems}, volume~13, pp.\
  409--415, Cambridge, Massachussetts, 2001. The MIT Press.

\bibitem[Chang \& Lin(2001)Chang and Lin]{Chang04}
Chang, C.-C. and Lin, C.-J.
\newblock {LIBSVM}: a library for support vector machines, 2001.
\newblock Software available at \url{http://www.csie.ntu.edu.tw/~cjlin/libsvm}.

\bibitem[Efron et~al.(2004)Efron, Hastie, Johnstone, and Tibshirani]{Efron04}
Efron, B., Hastie, T., Johnstone, L., and Tibshirani, R.
\newblock Least angle regression.
\newblock \emph{Annals of Statistics}, 32\penalty0 (2):\penalty0 407--499,
  2004.

\bibitem[Fiacco(1976)]{Fiacco76}
Fiacco, A.~V.
\newblock Sensitivity analysis for nonlinear programming using penalty methods.
\newblock \emph{Mathematical Programming}, 10\penalty0 (3):\penalty0 287--311,
  1976.

\bibitem[Friedman et~al.(2007)Friedman, Hastie, H\"{o}fling, and
  Tibshirani]{Friedman07}
Friedman, J., Hastie, T., H\"{o}fling, H., and Tibshirani, R.
\newblock {Pathwise coordinate optimization}.
\newblock \emph{Annals of Applied Statistics}, 1\penalty0 (2):\penalty0
  302--332, 2007.

\bibitem[G{\"a}rtner et~al.(2009)G{\"a}rtner, Giesen, and Jaggi]{Gartner10}
G{\"a}rtner, B., Giesen, J., and Jaggi, M.
\newblock An exponential lower bound on the complexity of regularization paths.
\newblock \emph{CoRR}, abs/0903.4817, 2009.

\bibitem[Giesen et~al.(2010)Giesen, Jaggi, and Laue]{Giesen10}
Giesen, J., Jaggi, M., and Laue, S.
\newblock Approximating parameterized convex optimization problems.
\newblock In de~Berg, Mark and Meyer, Ulrich (eds.), \emph{18th European
  Symposium on Algorithms}, volume 6346 of \emph{Lecture Notes in Computer
  Science}, pp.\  524--535. Springer Berlin / Heidelberg, 2010.

\bibitem[Hastie et~al.(2004)Hastie, Rosset, Tibshirani, and Zhu]{Hastie04}
Hastie, T., Rosset, S., Tibshirani, R., and Zhu, J.
\newblock The entire regularization path for the support vector machine.
\newblock \emph{Journal of Machine Learning Research}, 5:\penalty0 1391--1415,
  2004.

\bibitem[Platt(1999)]{Platt99}
Platt, J.~C.
\newblock Fast training of support vector machines using sequential minimal
  optimization.
\newblock In Sch\"{o}lkopf, Bernhard, Burges, Christopher J.~C., and Smola,
  Alexander~J. (eds.), \emph{Advances in Kernel Methods --- Support Vector
  Learning}, pp.\  185--208, Cambridge, MA, 1999. {MIT} Press.

\bibitem[Ritter(1984)]{Ritter84}
Ritter, K.
\newblock On parametric linear and quadratic programming problems.
\newblock In Cottle, R., Kelmanson, M.~L., and Korte, B. (eds.),
  \emph{Mathematical Programming: Proceedings of the International Congress on
  Mathematical Programming}, pp.\  307--335. Elsevier Science Publisher B.V.,
  1984.

\end{thebibliography}
\bibliographystyle{icml2011}
}

\makeatletter
 \renewcommand{\thefigure}{%
 \thesection.\arabic{figure}}
 \@addtoreset{figure}{section}
\makeatother
\numberwithin{equation}{section}
\renewcommand{\theequation}{\Alph{section}.\arabic{equation}}

\appendix
\section*{Appendix}

\red{Here, we provide proofs of Theorems 1, 2 and 
simplified formulation of the optimization problem (\ref{eq:partition.qp}). }

\section{Proof of Theorem 1}
\label{sec:proof_thm1}

Here, we provide a proof of the following theorem:

\setcounter{thm}{0}
\begin{thm} \label{thm:}
Let 
$\pi = (\cO, \cM, \cI)$ 
be the partition at the optimal solution at 
$\theta$
and assume that
\begin{eqnarray*}
\vM =
  \begin{bmatrix}
   0 & \vy_{\cM}^\top \\
   \vy_{\cM} & \vQ_{\cM} 
  \end{bmatrix}
\end{eqnarray*}
is non-singular.
Then, as long as $\pi$ is unchanged,
$\{ \beta_i \}_{i = 0}^n$ is given by
\begin{align}
 \begin{aligned}
 \begin{bmatrix}
  \beta_0 \\
  \vbeta_\cM
 \end{bmatrix} 
 & = 
 - \vM^{-1}
 \begin{bmatrix}
  \vy_{\cal I}^\top \\
  \vQ_{{\cal M},{\cal I}} 
 \end{bmatrix}
 \mb{d}_\cI, 
 \\  
 \vbeta_\cO &= \bm{0}, 
 \\\  
 \vbeta_\cI &= \mb{d}_\cI.  
 \end{aligned}
\label{eq:apx_beta}
\end{align}
\end{thm}

\begin{proof}
As long as $\pi$ is unchanged,
$\alpha_i$ for $i \in \cO$ and $i \in \cI$
must be
\begin{eqnarray*}
 \alpha_i &=& 0, \ i \in \cO, \\
 \alpha_i &=& C_i^{(\theta)}, \ i \in \cI. 
\end{eqnarray*}
Therefore, we see that
$\vbeta_\cO = \mb{0}$ and $\vbeta_\cI = \vd_\cI$.
From the definition of $\cM$, 
at the optimal, the following linear system holds
\begin{eqnarray*}
 \vQ_\cM \valpha^{(\theta)}_\cM + 
  \vQ_{\cM, \cI} \vc^{(\theta)}_\cI + \vy_\cM \alpha_0^{(\theta)} = \mb{1}.
\end{eqnarray*}
Combining with the equality constraint of the dual problem 
$\vy^\top \valpha = 0$,
we obtain the following linear system:
\begin{eqnarray*}
 \vM 
  \begin{bmatrix}
   \alpha^{(\theta)}_0 \\
   \valpha^{(\theta)}_\cM
  \end{bmatrix}
  +
  \begin{bmatrix}
   \vy_\cI^\top \\
   \vQ_{\cM, \cI}
  \end{bmatrix}
  \vc_\cI^{(\theta)}
  =
  \begin{bmatrix}
   0 \\
   \mb{1}
  \end{bmatrix}.
\end{eqnarray*}
Solving this, we obtain
\begin{eqnarray*}
 \begin{bmatrix}
   \alpha^{(\theta)}_0 \\
   \valpha^{(\theta)}_\cM
 \end{bmatrix}
 = - \vM^{-1}
   \begin{bmatrix}
   \vy_\cI^\top \\
   \vQ_{\cM, \cI}
  \end{bmatrix}
  \vc_\cI^{(\theta)}
  + \vM^{-1}
  \begin{bmatrix}
   0 \\
   \mb{1}
  \end{bmatrix}.
\end{eqnarray*}
Using 
$\vc^{(\theta + \Delta \theta)} = \vc^{(\theta)} + \theta \vd$,
we can write
\begin{eqnarray*}
 \begin{bmatrix}
  \alpha^{(\theta + \Delta \theta)}_0 \\
  \valpha^{(\theta + \Delta \theta)}_\cM
 \end{bmatrix}
 =
 \begin{bmatrix}
   \alpha^{(\theta)}_0 \\
   \valpha^{(\theta)}_\cM
 \end{bmatrix}
 - \theta
\vM^{-1}
 \begin{bmatrix}
  \vy_{\cal I}^\top \\
  \vQ_{{\cal M},{\cal I}} 
 \end{bmatrix}
 \mb{d}_\cI
\end{eqnarray*}
Then, we obtain (\ref{eq:apx_beta}).
\end{proof}

\section{Proof of Theorem 2}
\label{sec:proof_thm2}

Here, we provide a proof of Theorem 2.
First, we prove the following lemma.

\begin{lem} \label{lem:}
 Suppose $\whvbeta \in \RR^n$, $\whbeta_0 \in \RR$ and $\whvg = \vQ \whvbeta + \vy \whbeta_0$
 satisfy the following conditions:
\begin{subequations}
 \begin{align}
  \whg_i \whbeta_i = 0, \
  \whg_i \geq 0, \
  \whbeta_i \geq 0, \
  i \in \cB_\cO,  
  \label{eq:apx_compl.cond1} \\
  \whg_i (d_i - \whbeta_i)  = 0, \
  \whg_i \leq 0, \
  \whbeta_i \leq d_i, \
  i \in \cB_\cI,  
  \label{eq:apx_compl.cond2} \\
  \whvg_{\cM_{k+\frac{1}{2}}} = \mb{0}, \
  \whvbeta_{\cO_{k+\frac{1}{2}}} = \mb{0}, \
  \whvbeta_{\cI_{k+\frac{1}{2}}} = \mb{d}_\cI, \\
   \vy^\top \whvbeta = 0. \label{eq:apx_const.eq.grad} 
 \end{align} 
  \label{eq:apx_grad.cond}
\end{subequations}
 Then,
 $\whbeta_0$,
 $\whvbeta$ and
 $\whvg$
 are equal to 
 $\beta_0$, 
 $\vbeta$ and
 $\vg$, respectively, where 
 $\pi$ is determined by 
 the update rule
\begin{align}
\begin{aligned}
 \cM_k = \ & \cM_{k+\frac{1}{2}} \ \cup \ 
 \{ i \mid \beta_i > 0, g_i = 0, i \in \cB_\cO \} \\ 
 & \cup \ 
 \{ i \mid \beta_i < d_i, g_i = 0, i \in \cB_\cI \}, \\
 \cO_k = \ & \cO_{k+\frac{1}{2}} \ \cup \ 
 \{ i \mid \beta_i = 0, g_i \geq 0, i \in \cB_\cO \}, \\
 \cI_k = \ & \cI_{k+\frac{1}{2}} \ \cup \ 
 \{ i \mid \beta_i = d_i, g_i \leq 0, i \in \cB_\cI \},
\end{aligned}
 \label{eq:apx_update.pi}
\end{align}
 using $\whvbeta$ and $\whvg$.
\end{lem}

\begin{proof}
 Since the conditions (\ref{eq:apx_compl.cond1}) and (\ref{eq:apx_compl.cond2})
 hold, all of the elements of $\cB$ is assigned to one of the three
 index sets by (\ref{eq:apx_update.pi}).
 From the definitions of 
 $\cM_{k+1}$, $\cO_{k+1}$ and $\cI_{k+1}$ (\ref{eq:apx_update.pi}),
 we see
 $\whvg_{\cM_{k+1}} = \mb{0}$,
 $\whvbeta_{\cO_{k+1}} = \mb{0}$ and
 $\whvbeta_{\cI_{k+1}} = \vd_{\cI_{k+1}}$.
 Using these three equations and (\ref{eq:apx_const.eq.grad}),
 we can easily obtain the same linear system as (\ref{eq:apx_beta}).
\end{proof}

Next, we consider theorem 2.

\begin{thm} \label{thm:}
Let 
$\whbeta_0$,
$\whvbeta$ and
$\whvg$ 
be the optimal solutions of the following QP problem:
\begin{eqnarray}
 \min_{\whbeta_0, \whvbeta, \whvg}  &&
  \sum_{i \in \cB_\cO} \whg_i \whbeta_i + 
  \sum_{i \in \cB_\cI} \whg_i (\whbeta_i - d_i) 
  \label{eq:apx_partition.qp} \\
 {\rm s.t.} & & \left\{ 
  \begin{aligned} 
  &  \whvg_{\cB_\cO} \geq \bm{0}, \ \whvbeta_{\cB_\cO} \geq \bm{0}, 
  \ \whvg_{\cB_\cI} \leq \bm{0}, \ \whvbeta_{\cB_\cI} \leq \mb{d}_\cI,
   \nonumber \\
  &  \whvg_{\cM_{k+\frac{1}{2}}} = \bm{0}, \
   \whvbeta_{\cO_{k+\frac{1}{2}}} = \bm{0}, \
   \whvbeta_{\cI_{k+\frac{1}{2}}} = \mb{d}_{\cI_{k+\frac{1}{2}}}, 
   \nonumber \\
  & \vy^\top \whvbeta = 0, \ 
   \whvg = \vQ \whvbeta + \vy \whbeta_0, \nonumber 
  \end{aligned} \right. \nonumber 
\end{eqnarray}
and $\pi$ is determined by (\ref{eq:apx_update.pi}) using 
$\whvbeta$ and $\whvg$.
Then
$\whbeta_0$,
$\whvbeta$ and
$\whvg$ 
satisfy (\ref{eq:apx_grad.cond}) and they are equal to the gradient 
$\beta_0$,
$\vbeta$ and
$\vg$, respectively.
\end{thm}


\begin{proof}
 In this proof, we omit subscript $k+\frac{1}{2}$ 
 to simplify the notation.
 First, we rewrite the optimization problem
 (\ref{eq:apx_partition.qp}) as follows:
 \begin{eqnarray*}
 \min_{\whbeta_0, \whvbeta, \whvg} &&
  \sum_{i \in \cB_\cO \cup \cM \cup \cO} g_i \whbeta_i + 
  \sum_{i \in \cB_\cI \cup \cI} g_i (\whbeta_i - d_i) \label{eq:apx_partition.qp2}\\
 {\rm s.t.} && 
  \left\{
   \begin{aligned}
    &  \whvg = \vQ \whvbeta + \vy \whbeta_0, \nonumber \\
    &  \whvg_{\cB_\cO} \geq \bm{0}, \ \whvbeta_{\cB_\cO} \geq 0, \nonumber \\
    &  \whvg_{\cB_\cI} \leq \bm{0}, \ \whvbeta_{\cB_\cI} \leq \mb{d}_\cI, \nonumber \\
    &  \whvg_\cM = \bm{0}, \ \whvbeta_\cO = \bm{0}, \ \whvbeta_\cI = \mb{d}_\cI, \nonumber \\
    &  \vy^\top \whvbeta = 0.\nonumber 
   \end{aligned}
\right.
\end{eqnarray*} 
Although we slightly modified the expression of the objective function,
its value is the same as (\ref{eq:apx_partition.qp})
as long as the equality constraints hold.
From the inequality constraints, we see that the objective value
 is always non-negative in the feasible region.

To simplify the notation, we introduce the following new variables:
\begin{eqnarray*}
 \tvbeta &=& \mb{E}^{(2)} \mb{d} + \mb{E} \whvbeta, \
  \tvg = \mb{E} \whvg, \\
 \tvy &=& \mb{E} \vy, \\
 \tvQ &=& \mb{E} \vQ \mb{E},
\end{eqnarray*}
where
\begin{eqnarray*}
 E^{(1)}_{ij} &=&
  \left\{
   \begin{array}{ll}
    1 & \text{ for } \{(i,j) \mid i = j, i \in \cB_\cO \cup \cM \cup \cO \}, \\
    0 & \text{ others },
   \end{array}
  \right. \\
 E^{(2)}_{ij} &=&
  \left\{
   \begin{array}{ll}
    1 & \text{ for } \{(i,j) \mid i = j, i \in  \cB_\cI \cup \cI \}, \\
    0 & \text{ others },
   \end{array}
  \right.\\
 \mb{E} &=& \mb{E}^{(1)} - \mb{E}^{(2)}.
\end{eqnarray*}
Moreover, if we set $\cT = \cO \cup \cI$, the optimization 
problem (\ref{eq:apx_partition.qp}) is written as
\begin{eqnarray*}
 \min_{\beta_0, \tvbeta, \tvg} && 
  \tvbeta^\top \tvQ \tvbeta + r_0 \beta_0 + \mb{r}^\top \tvbeta \\
 {\rm s.t.} && 
  \left\{
   \begin{aligned}  
  &  \tilde{\vQ} \tilde{\vbeta} + \tilde{\vy} \beta_0 + \mb{r} - \tilde{\mb{g}}
  = \bm{0}, \\
  &  \tilde{\mb{g}}_\cM = \bm{0}, \tilde{\mb{g}}_\cB \geq \bm{0}, \\
  &  \tilde{\vbeta}_\cT = \bm{0}, \tilde{\vbeta}_\cB \geq \bm{0}, \\
  &  \tilde{\vy}^\top \tilde{\vbeta} = r_0,
   \end{aligned}
  \right.
\end{eqnarray*}
where
\begin{eqnarray*}
 \mb{r} &=& - \tilde{\vQ}_{:,\cI} \mb{d}_\cI - \tilde{\vQ}_{:,\cB_cI}
  \mb{d}_{\cB_\cI}, \
 r_0 = - \tilde{\vy}_\cI^\top \mb{d}_\cI - \tilde{\vy}_{\cB_cI} \mb{d}_{\cB_\cI},
\end{eqnarray*} 
are constants.
Let
$\vxi \in \RR^n, \vmu_\cM \in \RR^{|\cM|}, \vmu_\cB \in
 \RR^{|\cB|}, \vnu_\cT \in \RR^{|\cT|}, \vnu_\cB \in
 \RR^{|\cB|}, \rho \in \RR$
be the Lagrange multipliers.
Then, the Lagrangian is
\begin{eqnarray*}
 L &=& 
  \tvbeta^\top \tvQ \tvbeta + r_0 \beta_0 + \mb{r}^\top \tvbeta 
  + \mb{\xi}^\top \left(
		     \tilde{\vQ} \tilde{\vbeta} + \tilde{\vy} \beta_0 + \mb{r} - \tilde{\mb{g}}
		    \right) \\
 && + \vmu_\cM^\top \tilde{\mb{g}}_\cM - \vmu_\cB^\top \tilde{\mb{g}}_\cB 
  + \vnu_\cT^\top \tilde{\vbeta}_\cT - \vnu_\cB^\top \tilde{\vbeta}_\cB \\
 && + \rho \left( \tilde{\vy}^\top \tilde{\vbeta} - r_0 \right),
\end{eqnarray*}
where 
$\vmu_\cB \geq \bm{0}$, 
$\vnu_\cB \geq \bm{0}$.
Differentiating $L$, we obtain
\begin{eqnarray*}
 \pd{L}{\tvbeta} &=&
  2 \tvQ \tvbeta + \mb{r} + \tvQ \mb{\xi} + \tilde{\vnu} + \rho \tvy = \bm{0},\\
 \pd{L}{\beta_0} &=&
  r_0 + \mb{\xi}^\top \tvy = \bm{0}, \\
 \pd{L}{\tvg} &=& - \mb{\xi} + \tilde{\vmu} = 0,
\end{eqnarray*}
where
$\tilde{\vnu} \in \RR^n$
 is a vector whose components are
$\tilde{\vnu}_\cM = \bm{0}$, 
$\tilde{\vnu}_\cB = - \vnu_\cB$, 
$\tilde{\vnu}_\cT = \vnu_\cT$ and
$\tilde{\mb{\mu}} \in \RR^n$ has
$\tilde{\mb{\mu}}_\cM = \mb{\mu}_\cM$, 
$\tilde{\mb{\mu}}_\cB = - \mb{\mu}_\cB$, 
$\tilde{\mb{\mu}}_\cT = \bm{0}$.
Using these equations, we obtain the following dual problem: 
\begin{eqnarray}
 \max_{\tvbeta,\mb{\xi},\tilde{\mb{\nu}},\tilde{\mb{\mu}},\rho} &&
  - \tvbeta^\top \tvQ \tvbeta - \rho r_0 + \mb{\xi}^\top \mb{r}
  \label{eq:apx_partition.qp.dual.obj} \\
 {\rm s.t.} && \left\{
 \begin{aligned}
  & 2 \tvQ \tvbeta + \mb{r} + \tvQ \mb{\xi} + \tilde{\mb{\nu}} + \rho \tvy
  = \bm{0},
  \\
  & r_0 + \mb{\xi}^\top \tvy = \bm{0}, 
  \\ 
  & - \mb{\xi} + \tilde{\mb{\mu}} = 0, 
  \\
  & \tilde{\mb{\nu}}_\cM = \bm{0}, \ \tilde{\mb{\nu}}_\cB \leq \bm{0}, 
  \\
  & \tilde{\mb{\mu}}_\cT = \bm{0}, \ \tilde{\mb{\mu}}_\cB \leq \bm{0}.
 \end{aligned}
 \right.
  \label{eq:apx_partition.qp.dual.const} 
\end{eqnarray}
Using the constraints of this problem (\ref{eq:apx_partition.qp.dual.const}),
we can derive the following bound of the objective function (\ref{eq:apx_partition.qp.dual.obj}):
\begin{align*}
& - \tvbeta^\top \tvQ \tvbeta - \rho r_0 + \mb{\xi}^\top \mb{r} \\
  &=  - \tvbeta^\top \tvQ \tvbeta + \rho \mb{\xi}^\top \tvy + \mb{\xi}^\top \mb{r} \\
  &=  - \tvbeta^\top \tvQ \tvbeta 
  - 2 \tvbeta^\top \tvQ \mb{\xi} - \mb{\xi}^\top \tvQ \mb{\xi} - \mb{\xi}^\top \tilde{\mb{\nu}}  \\
 &= - (\tvbeta + \mb{\xi})^\top \tvQ (\tvbeta + \mb{\xi})
  - \mb{\xi}^\top \tilde{\mb{\nu}} \\
 &= - (\tvbeta + \mb{\xi})^\top \tvQ (\tvbeta + \mb{\xi})
  - \tilde{\mb{\mu}}^\top \tilde{\mb{\nu}} \leq 0.
\end{align*}
From this we see that the dual objective function is less than or equal
 to $0$.
Thus, the optimal objective value of the optimization problem is $0$.
Then the conditions (\ref{eq:apx_grad.cond}) is satisfied.
From lemma~$1$, the claim is proved.
\end{proof}

\section{Reformulate the Optimization Problem (10)}
\label{sec:reduced_QP}

We reformulate the optimization problem (10) to reduce the number of
variables and constraints.
Here again, 
we omit subscript of $\cM$, $\cO$ and $\cI$
to simplify the notation.

Define $\cB = \{ b_1, \ldots, b_{|\cB|} \}$,
$\cS_\cO = \{ i \in \{1, \ldots, |\cB| \} \mid b_i \in \cB_\cO \}$
and
$\cS_\cI = \{ i \in \{1, \ldots, |\cB| \} \mid b_i \in \cB_\cI \}$.
When $|\cB| \neq 0$, 
the optimization problem (10) can be re-formulated as
\begin{eqnarray*}
 \min_{\vbeta_\cB} & &
 \vbeta_\cB^\top \vQ' \vbeta_\cB + 
  (\mb{v}_\cB - \vQ'_{:,\cS_\cI} \mb{d}_{\cB_\cI})^\top \vbeta_\cB \\
 {\rm s.t.} & & 
 \left\{
 \begin{aligned}
  & \vQ'_{\cS_\cO,:} \vbeta_\cB + \mb{v}_{\cB_\cO} \geq \mb{0}, \\
  & \vQ'_{\cS_\cI,:} \vbeta_\cB + \mb{v}_{\cB_\cI} \leq \mb{0}, \\
  & \vbeta_{\cB_\cO} \geq \mb{0}, \ \vbeta_{\cB_\cI} \leq \mb{d}_{\cB_\cI},
 \end{aligned}
 \right.
\end{eqnarray*}
where
\begin{eqnarray*}
 \vQ' &=& \vQ_\cB 
  - 
  \begin{bmatrix}
   \vy_\cB & \vQ_{\cB,\cM} 
  \end{bmatrix}
  \vM^{-1}
  \begin{bmatrix}
   \vy_\cB^\top \\
   \vQ_{\cB,\cM}    
  \end{bmatrix}, \\
 \mb{u} &=& - \vM^{-1} 
  \begin{bmatrix}
   \vy_\cI^\top \\
   \vQ_{\cM, \cI} 
  \end{bmatrix}
  \vd_\cI, \\
 \mb{v} &=& 
 \begin{bmatrix}
  \vy & \vQ_{:,\cM} 
 \end{bmatrix}
 \mb{u} + \vQ_{:,\cI} \vd_\cI.
\end{eqnarray*}
On the other hand, when $|\cB| = 0$, (10) becomes
\begin{eqnarray*}
 \min_{\vbeta_\cB, \beta_0} &&
 \vbeta_\cB^\top \vQ_\cB \vbeta_\cB + 
  (\vQ_{\cB,\cI} \mb{d}_\cI - \vQ_{\cB,\cB_\cI} \mb{d}_{\cB_\cI})^\top \vbeta_\cB \\
 && - (\vy_\cI^\top \mb{d}_\cI + \vy_{\cB_\cI}^\top \mb{d}_{\cB_\cI}) \beta_0 \\
 {\rm s.t.}  && 
 \left\{
 \begin{aligned}
  & \vy_\cB^\top \vbeta_\cB + \vy_\cI^\top \mb{d}_\cI = 0 \\
  & \vQ_{\cB_\cO, \cB} \vbeta_\cB + \vQ_{\cB_\cO, \cI} \mb{d}_\cI + \vy_{\cB_\cO} \beta_0 \geq \mb{0} \\
  & \vQ_{\cB_\cI, \cB} \vbeta_\cB + \vQ_{\cB_\cI, \cI} \mb{d}_\cI + \vy_{\cB_\cI} \beta_0 \leq \mb{0} \\
  & \vbeta_{\cB_\cO} \geq \mb{0}, \  \vbeta_{\cB_\cI} \leq \mb{d}_{\cB_\cI}.
 \end{aligned}
 \right.
\end{eqnarray*}






\end{document}